%% file: main.tex
\documentclass[11pt,twoside]{article}

\usepackage{fancyhdr}
\usepackage[colorlinks,citecolor=blue,urlcolor=blue,linkcolor=blue,bookmarks=false]{hyperref}
\usepackage{amsfonts,epsfig,graphicx}
\usepackage{afterpage}
\usepackage{amsmath,amssymb,amsthm} 
\usepackage{fullpage}
\usepackage[T1]{fontenc} 
\usepackage{epsf} 
\usepackage{graphics} 
\usepackage{amsfonts,amsmath}
\usepackage[sort,numbers]{natbib} 
\usepackage{psfrag,xspace}
\usepackage{color,etoolbox}

\usepackage{url}
\usepackage[inline]{enumitem}
\usepackage{graphicx} 
\usepackage{pdfpages}
\usepackage{subfigure} 
\usepackage{verbatim}
\usepackage{booktabs}

\usepackage{algorithm}
\usepackage{algpseudocode}
\usepackage{wrapfig}
\usepackage{lipsum}

\usepackage{amsthm}
\usepackage{mysymbols}

\usepackage{mysymbols}
\newcommand{\tparam}{{\theta^*}}

\newcommand{\eparam}{{\widehat{\theta}}}

\newcommand{\eSRM}{{\widehat{\theta}_{\text{SRM}}}}

\def\grad{\nabla}
\newcommand{\trueDist}{{{P}^{*}}}

\newcommand{\indic}[1]{ \mathbb{I} \tlprn {#1}}

\def\iid{i.i.d}

\newcommand{\opt}{{\text{OPT}}}

\begin{document}

\begin{center} 
{\LARGE{\bf{A Unified Approach to Robust Mean Estimation}}}

\vspace*{.3in}

{\large{
\begin{tabular}{ccccc}
Adarsh Prasad$^\ddagger$~~~Sivaraman Balakrishnan$^\dagger$~~~Pradeep Ravikumar$^\ddagger$ \\
\end{tabular}

\vspace*{.1in}

\begin{tabular}{ccc}
Machine Learning Department$^{\ddagger}$ \\
Department of Statistics and Data Science$^{\dagger}$ \\
\end{tabular}

\begin{tabular}{c}
Carnegie Mellon University, \\
Pittsburgh, PA 15213.
\end{tabular}

\vspace*{.2in}

}}
\begin{abstract}
  In this paper, we develop connections between two seemingly disparate, but central, models in robust statistics: Huber's $\epsilon$-contamination model and the heavy-tailed noise model.
  We provide conditions under which this connection provides near-statistically-optimal estimators. 
  Building on this connection, we provide a simple variant of recent computationally-efficient algorithms for mean estimation in Huber's model, which given our connection entails that the same  efficient sample-pruning based estimators is simultaneously robust to heavy-tailed noise and Huber contamination.
  Furthermore, we complement our efficient algorithms with statistically-optimal albeit computationally intractable estimators, which are simultaneously optimally robust in both models. We study the empirical performance of our proposed estimators on synthetic datasets, and find that our methods convincingly outperform a variety of practical baselines.
\end{abstract}

\end{center}

\section{Introduction}
Modern data sets that arise in various branches of science and engineering are characterized by their ever increasing scale and richness. This is spurred in part by easier access to computer, internet and various sensor-based technologies that enable the collection of such varied datasets. But on the flip side, these large and rich data-sets are no longer carefully curated, are often collected in a decentralized, distributed fashion, and consequently are plagued with the complexities of heterogeneity, adversarial manipulations, and outliers. The analysis of these huge datasets is thus fraught with methodological challenges.

To understand the fundamental challenges and tradeoffs in handling such ``dirty data'' is precisely the premise of the field of robust statistics. Here, the aforementioned complexities are largely formalized under two different models of robustness: (1) {\bf The heavy-tailed model:} Here the sampling distribution can have thick tails, for instance, only low-order moments of the distribution are assumed to be finite; and (2) {\bf The $\epsilon$-contamination model: } Here the sampling distribution is modeled as a well-behaved distribution contaminated by an $\epsilon$ fraction of arbitrary outliers. In each case, classical estimators of the distribution (based for instance on the maximum likelihood estimator) can behave considerably worse (potentially arbitrarily worse) than under standard settings where the data is better behaved, satisfying various  regularity properties. In particular, these classical estimators can be extremely sensitive to the tails of the distribution or to the outliers, so that the broad goal in robust statistics is to construct estimators that improve on these classical estimators by reducing their sensitivity to outliers. 



\paragraph{Heavy Tailed Model.} Concretely, focusing on the fundamental problem of robust mean estimation, in the heavy tailed model we observe $n$ samples $x_1,\ldots, x_n$ drawn independently from a distribution $P$, which is only assumed to have low-order moments finite (for instance, $P$ only has finite variance). The goal of past work \cite{catoni2012challenging,minsker2015geometric,lugosi2017sub,catoni2017dimension} has been to design an estimator $\eparam_n$ of the true mean $\mu$ of $P$ which has a small $\ell_2$-error with high-probability. Formally, for a given $\delta > 0$, we would like an estimator with minimal $r_\delta$ such that,
\begin{align}\label{eqn:deviation}
    P(\norm{\eparam_n - \mu}{2} \leq r_\delta) \geq 1- \delta.
\end{align}
As a benchmark for estimators in the heavy-tailed model, we observe that when $P$ is the multivariate normal (or sub-Gaussian) distribution with mean $\mu$ and covariance $\Sigma$, it can be shown (see~\citep{hanson1971bound}) that the sample mean $\widehat{\mu}_n = (1/n) \sum_i x_i $ satisfies, with probability at least $1-\delta$\footnote{Here and throughout our paper we use the notation $\lesssim$ to denote an inequality with universal constants dropped for conciseness.},
\begin{align}\label{eqn:subGaussBound}
    \bignorm{\widehat{\mu}_n - \mu}{2} \lesssim \sqrt{\frac{\trace{\Sigma}}{n}} + \sqrt{\frac{\bignorm{\Sigma}{2} \log(1/\delta)}{n}}.
\end{align}

The bound is referred to as a \emph{sub-Gaussian}-style error bound. However, for heavy tailed distributions, as for instance showed in \cite{catoni2012challenging}, the sample mean only satisfies the sub-optimal bound  $r_{\delta} = \Omega(\sqrt{d/{n\delta}})$. Somewhat surprisingly recently work~\cite{lugosi2017sub} showed that the sub-Gaussian error bound is achievable while \emph{only assuming that $P$ has finite variance}, but by a carefully designed estimator.  In the univariate setting, the classical median-of-means estimator~\citep{alon96,nemirovski1983problem,jerrum86} and Catoni's M-estimator~\citep{catoni2012challenging} achieve this surprising result but designing such estimators in the multivariate setting has proved challenging. \citet{minsker2015geometric} proved results for the geometric median-of-means (GMOM), which, (1) partitions the data into $k = \ceil{3.5 \log (1/\delta)}$ blocks, (2) computes sample mean within each block $\{\widehat{\mu}_{i}\}_{i=1}^k$ and (3) and returns the geometric median $\eparam_{\text{MOM},\delta} = \argmin_{\theta} \sum_{i} \bignorm{\theta - \widehat{\mu}_i}{2}$. In particular, the paper~\cite{minsker2015geometric} showed that $\eparam_{\text{MOM},\delta}$ is such that, with probability at least $1-\delta$, 
\begin{align}
\label{eqn:gmom}
\norm{\eparam_{\text{MOM},\delta} - \mu }{2} \lesssim \sqrt{\frac{\trace{\Sigma} \log(1/\delta)}{n}}.
\end{align}
Note that the GMOM estimator does not match the true sub-Gaussian bounds~\eqref{eqn:subGaussBound}. 
Estimators that achieve truly sub-Gaussian bounds, but which are computationally intractable, were proposed recently by~\citet{lugosi2017sub} and subsequently \citet{catoni2017dimension}. \citet{hopkins2018sub} and later~\citet{cherapanamjeri2019fast} developed a sum-of-squares based relaxation of~\citet{lugosi2017sub}'s estimator, thereby giving a polynomial time algorithm which achieves optimal rates. However, while polynomial-time, these estimators are still far from being implementable and used in practice. In this paper, we propose and study \emph{practical} estimators that in some cases improve on the geometric median-of-means and in some cases are truly sub-Gaussian.

\paragraph{Huber's $\epsilon$-Contamination Model.} In this setting, instead of observing samples directly from the true distribution $P$, we observe samples drawn from $P_\epsilon$, which for an arbitrary distribution $Q$ is defined as a mixture model,
\begin{align}\label{eqn:huber_mixture}
    P_\epsilon = (1-\epsilon)P + \epsilon Q.
\end{align}
The distribution $Q$ allows one to model arbitrary outliers, which may correspond to gross corruptions, or subtle deviations from the true model. 
There has been a lot of classical work studying estimators in the $\epsilon$-contamination model under the umbrella of robust statistics (see for instance~\citep{hampel1986robust} and references therein). However, most of the estimators come that come with strong guarantees are computationally intractable~\citep{tukey1975mathematics}, while others are statistically sub-optimal heuristics~\citep{hastings1947low}. Recently, there has been substantial progress \citep{diakonikolas2016robust,lai2016agnostic,kothari2018robust,charikar2017learning,diakonikolas2017being,balakrishnan2017computationally} designing provably robust which are computationally tractable while achieving near-optimal contamination dependence (i.e. dependence on the fraction of outliers $\epsilon$) for computing means and covariances. In the Huber model, using information-theoretic lower bounds~\citep{chen2016general,lai2016agnostic,hopkins2018mixture}, it can be shown that any estimator must suffer a \emph{non-zero} bias (the asymptotic error as the number of samples go to infinity). For example, for the class of distributions with bounded variance, $\Sigma \precsim \sigma^2 \calI_p$, the bias is lower bounded by $\Omega(\sigma \sqrt{\epsilon})$. Surprisingly, the optimal bias that can be achieved is often independent of the data dimension. In other words, in many interesting cases optimally robust estimators in Huber's model can tolerate a constant fraction $\epsilon$ of outliers, \emph{independent of the dimension.}


Despite their apparent similarity, developments of estimators that are robust in each of these models for outliers,
have remained relatively independent. 
Focusing on mean estimation we notice subtle differences, in the heavy-tailed model our target is the mean of the sampling distribution whereas in the Huber model our target is the mean of the \emph{decontaminated} sampling distribution $P$. Beyond this distinction, it is also important to note that as highlighted above the natural focus in heavy-tailed mean estimation is on achieving strong, high-probability error guarantees, while in Huber's model the focus has been on achieving dimension independent bias.  

 While the aforementioned recent estimators for mean estimation under Huber contamination have a polynomial computational complexity, their corresponding sample complexities are only known to be \emph{polynomial} in the dimension $p$. For example,~\citet{kothari2018robust} and \citet{hopkins2018mixture} designed estimators which achieve optimal bias for distributions with \emph{certifiably} bounded $2k$-moments, but their statistical sample complexity scales as $O(p^{k})$. \citet{steinhardt2017resilience} studied mean estimation and presented an estimator which has a sample complexity of $\Omega\bigparen{{p^{1.5}}}$. 
\paragraph{Contributions.} In this work, we aim to bridge the gap between these two models of robustness. In particular, we show that it is possible to decompose any heavy-tailed distribution as a mixture of a well-behaved and contamination distribution. Further, to exploit this insight, we study mean estimators in the $\epsilon$-contamination model. Building on recent algorithmic developments we propose a computationally efficient estimator which achieves the bias-lower bound and simultaneously has small $\ell_2$ error with high-probability. Leveraging our earlier connection between the heavy-tailed and Huber models, and our computationally efficient estimator in Huber's model, we develop the \emph{first} sample-pruning based estimator for the heavy-tailed setting. This estimator in some settings achieves the optimal sub-Gaussian error bound and improves on the geometric median-of-means. Furthermore, the estimator we propose is easy to implement and practical, and we show its efficacy on a range of synthetic datasets. We complement our efficient estimator by designing statistically optimal , albeit computationally intractable estimators, which are optimally robust in both models simultaneously. In particular, our results show that there do exist estimators, which achieve optimal bias for general $2k$-moment bounded distributions, while having an optimal statistical sample complexity of $O(p)$.

\paragraph{Concurrent Works.} After the initial submission of this manuscript, we became aware of two concurrent works~\citep{lecue2019robust,dong2019quantum}.~\citet{lecue2019robust}  obtain a nearly-linear time algorithm for robust mean estimation for distributions with bounded covariance, which is simultaneously robust in both models. While our proposed estimator requires more assumptions, it is simpler and easier to implement. In particular, the algorithm of~\citep{lecue2019robust} relies on solvers for packing/covering semidefinite programs, which are not yet
practical. \citet{dong2019quantum} also propose a nearly-linear time algorithm for robust mean estimation. However, while their estimator achieves the optimal asymptotic bias in Huber's $\epsilon$-contamination, it achieves sub-optimal statistical sample complexity. In particular, at a given confidence level $\delta$, the error of their estimator scales as $O(\sqrt{\frac{1}{n \delta}})$ instead of $O(\sqrt{\frac{\log(1/\delta)}{n}})$.

\paragraph{Notation and some definitions.} Let $x$ be a random vector with mean $\mu$ and covariance $\Sigma$. We say that the $x$ has bounded $2k$-moments if for all $v \in \calS^{p-1}$, $ \Exp[(v^T(x - \mu))^{2k}] \leq C_{2k} \bigparen{\Exp[(v^T(x - \mu))^2]}^k$. We let,
\begin{align}
\label{eqn:optdef}
\opt_{n,\Sigma,\delta} \defeq  \sqrt{\frac{\trace{\Sigma}}{n}} + \sqrt{\frac{\bignorm{\Sigma}{2}\log(1/\delta)}{n}},
\end{align}
denote the sub-Gaussian deviation bound in \eqref{eqn:subGaussBound}, satisfied by the sample mean of a sub-Gaussian distribution, at a confidence level $\delta$. Let $r(\Sigma) \defeq \frac{\trace{\Sigma}}{\bignorm{\Sigma}{2}}$ be the \emph{effective rank} of $\Sigma$. Note that $ 1 \leq r(\Sigma) \leq r$, where $r$ is the rank of $\Sigma$. Throughout the paper, we use $c, c_1, c_2, \ldots, C, C_1, C_2, \ldots$ to denote positive universal constants.

\section{Oracle Mixture Model}\label{sec:oracle}
In this section, we motivate studying both models under a common lens by developing a decomposition of a heavy-tailed distribution $P$ into a well-behaved distribution $P_{\calO}$ and a contaminating distribution $Q$, i.e. we decompose $P$ as $P = (1-\epsilon_{\calO}) P_{\calO} + \epsilon_{\calO} Q$. 

At a high-level, our aim is to develop tightly concentrated estimators for the mean of the heavy-tailed distribution $P$. When the distribution $P$ is heavy-tailed the sample mean is not tightly concentrated because the extreme samples exert a strong influence on the higher moments of its error. Our strategy will be to prune these extreme samples and develop tightly concentrated estimators of the mean of $P_{\calO}$ (which will typically have lighter tails than $P$). However, this is only useful if the means of $P_{\calO}$ and $P$ are sufficiently close, and the crux of our estimator and its analysis, will be to carefully balance the degrading \emph{bias} from pruning samples with the improving concentration of our estimator for the mean of $P_{\calO}$.


To set the stage, in this section we develop the aforementioned decomposition and in Section~\ref{sec:effheavy} we study its algorithmic consequences.
Concretely, for any bounded variance distribution $P$ with mean $\mu$, we define an oracle $\calO:\real^p \mapsto \{0,1\}$ by 
\begin{align}
\label{eqn:oracledef}
\calO(x) = \indic{\bignorm{x - \mu}{2} \leq R}.
\end{align} We define $\calP_O$ to be the distribution of $P$ when conditioned on the event that $\calO(x) = 1$. Intuitively, samples that fall outside a radius $R$ of the mean are labeled as outliers.

Suppose given $n$ samples from $P$, our oracle estimator is the mean of the samples that belong to $P_{\calO}$: 
\begin{align*}
\widehat{\mu}_n = \Big(\sum \limits_{i=1}^n \calO(x_i) \Big)^{-1} \sum \limits_{i=1}^n 
    x_i \calO(x_i).
\end{align*}
The estimator is an oracle since we do not know the mean $\mu$. However, we will show in Section~\ref{sec:effheavy} that the \emph{existence} of an oracle is sufficient to design computationally and statistically effective estimators. 
Given $\delta \in (0.5,1)$, we consider values of $R$ and $n$ such that for a sufficiently small constant $c > 0$,
\begin{align}
\label{eqn:random}
\Big(\frac{\sqrt{\trace{\Sigma}}}{R}\Big)^{2k} + \frac{\log(1/\delta)}{n} \leq c,
\end{align}
and recall the definition of $\opt_{n,\Sigma,\delta}$ in~\eqref{eqn:optdef}. The following theorem characterizes the performance of $\widehat{\mu}_n$ as an estimator of the mean of $P$, effectively characterizing its bias and high-probability deviation. 


\begin{theorem}\label{thm:oracelArgument}
Let $P$ be any distribution with mean $\mu$ and covariance $\Sigma$ and bounded $2k$-moments for $k \in \{1,2\}$. 
For any $\delta \in (0.5,1)$ with probability at least $1 - \delta$,
\begin{align*}
    \norm{\widehat{\mu}_n - \mu}{2} \lesssim \opt_{n,\Sigma,\delta} +   \frac{R \log(1/\delta)}{n} + \bignorm{\Sigma}{2}^{\half} \Big(\frac{\sqrt{\trace{\Sigma}}}{ R }\Big)^{2k-1}.
\end{align*}
\end{theorem}
In essence, we need to choose the radius parameter $R$ to balance the deviation and bias terms above. This 
leads to a family of bounds for different choices of the radius parameter $R$ and the number $2k$ of bounded moments of $P$, which we summarize below. We begin by giving results for $k = 2$, i.e. when $P$ has bounded 4-th moment. Recalling the definition of the effective rank $r(\Sigma)$: for the remainder of this section, we suppose that $n$ is large enough so that for a sufficiently small constant $c > 0$, we have
that $\sqrt{r(\Sigma)} \leq cn/\log(1/\delta).$ Then we have the following Corollary.
\begin{corollary}\label{cor:oracle4}
  Suppose that $k = 2$, and that
$R$ is chosen as $\frac{\sqrt{\trace{\Sigma}}}{\sqrt{r(\Sigma)}^{1/4} \bigparen{\frac{\log(1/\delta)}{n}}^{1/4}},$ 
 then with probability at least $1 - \delta$, 
    \begin{align*}
     \norm{\widehat{\mu}_n - \mu}{2} \lesssim \opt_{n,\Sigma,\delta} +
    \frac{\sqrt{\trace{\Sigma}}}{\sqrt{r(\Sigma)}^{1/4}} \bigparen{\frac{\log(1/\delta}{n}}^{3/4}.
    \end{align*}
\end{corollary}
{\bf Remark: } If $n$ is large enough so that for a small constant $c > 0$ the effective rank is bounded as $r(\Sigma)^{3/2}  < cn/\log(1/\delta)$, then the first term in the bound dominates and the oracle based mean estimator achieves \emph{the optimal sub-Gaussian deviation bound.} This result suggests that for distributions with bounded 4-th moment, under a relatively mild assumption on the effective rank, we can estimate the mean optimally, by sample pruning: i.e. by computing the mean of the distribution $P_{\calO}$, discarding samples from the distribution $Q$.

Next, we show a similar result for distributions with bounded 2nd moment. 
\begin{corollary}\label{cor:oracle2}
    Suppose that $k=1$ and $R$ is chosen as $\frac{\sqrt{\trace{\Sigma}}}{{r(\Sigma)^{1/4}}\bigparen{\frac{\log(1/\delta)}{n}}^{1/2}}$,
    the with probability at least $1 - \delta$, 
    \begin{align*}\norm{ \widehat{\mu}_n- \mu}{2} \lesssim \opt_{n,\delta,\Sigma} + \bignorm{\Sigma}{2}^{\half} r(\Sigma)^{1/4} \sqrt{\frac{\log(1/\delta)}{n}}.\end{align*}
\end{corollary}
{\bf Remark: } When only the variance of the distribution $P$ is bounded our sample-pruning oracle does not achieve the optimal sub-Gaussian rate. However, since $\bignorm{\Sigma}{2}^{\half} r(\Sigma)^{1/4} \leq \sqrt{\trace{\Sigma}}$, comparing to the guarantee in~\eqref{eqn:gmom} we see that the (oracle) sample pruning estimator improves on the geometric median-of-means estimator. For instance, when $\Sigma = \sigma^2 \calI_p$ the deviations of our oracle estimator scales as $\sqrt{p^{1/2} \log(1/\delta)/n}$ while the deviations of the geometric median-of-means scales as $\sqrt{p \log(1/\delta)/n}$.

Taken together these results suggest that if we can solve the Huber mean estimation problem optimally, then, we can get near-optimal rates for heavy-tailed mean estimation.

\section{Efficient Estimators}
\label{sec:effheavy}
In this section, we show that as long as the oracles satisfy some nice properties, we can design computationally and statistically effective estimators. To be precise, suppose we are given set of $n$ samples $S$ from a distribution from $P$, and there exists an oracle subset of points $G$, then, our goal is to design an algorithm which can estimate the mean of points in $G$. As we show next, under mild assumptions and with some-additional side information, there exists a computationally efficient algorithm which returns an estimate close to the true mean of the points in G. 

Our algorithm is primarily based on the SVD-based filtering algorithm, which has appeared in different forms~\citep{klivans2009learning,awasthi2014power} and was recently reused by \citep{diakonikolas2016robust,diakonikolas2017being} for robust mean estimation. However, the previous versions and their analysis, while suited to bounds on the expected deviation, do not give tight high-probability non-asymptotic rates. Our estimator is presented in Algorithm~\ref{algo:filteringpD}. It proceeds in an iterative fashion, by (1) computing the principal eigenvector of the empirical covariance matrix, (2) projecting points along the the principal eigenvector, and (3) randomly sampling points based on their projection scores. This procedure is repeated until the operator norm of empirical covariance matrix is close to a known-upper bound of the operator norm of the good-set.  

\citet{diakonikolas2017being} follow a similar procedure, but remove a subset of points at a step, depending on if their projection score is above or below a randomly chosen threshold. While only a modest difference from ours, deriving high-probability results for their algorithm is not clear, and in particular, the bounds provided by \citep{diakonikolas2017being} are in expectation. In contrast, our variant of this iterative sample-and-remove procedure allows us to borrow tools from martingale analysis~\citep{xu2013outlier,liu2018high}, and we are able to get tight non-asymptotic high-probability bounds for mean estimation.

\begin{algorithm}[htbp]

\centering
\caption{Empirical Multivariate Filtering Estimator}
        \begin{algorithmic}
        \small{
\Function{FilterpD(S = $\{z_i\}_{i = 1}^n$, Upper Bound on $\norm{\Sigma_{G^0}}{2}$)}{}
\State Let $\eparam_S = \frac{1}{|S|}\sum_{i=1}^{|S|} z_i$ be the sample mean.
\State Let $\Sigma_{S} = \frac{1}{|S|}\sum_{i=1}^{|S|} (z_i - \eparam_S)(z_i - \eparam_S)^T$ be the sample covariance matrix.
\State Let $(\lambda,v)$ be the largest eigenvalue,eigenvector of $\Sigma_S$.
 \If {$\lambda < 32 \norm{\Sigma_{G^0}}{2}$}
    \State \Return $\eparam_S$
 \Else
    \State For each $z_i$, let $\tau_i \defeq \bigparen{v^T(z_i - \eparam_S)}^2$ to be its \emph{score}
    \State Randomly sample a point $z$ from $S$ according to 
    \[ \Pr(z_i~\text{chosen}) = \frac{\tau_i}{\sum_{j} \tau_j} \]
    \State \Return {\sc FilterpD}($S \backslash \{ z \}$ , $\norm{\Sigma_{G^0}}{2}$) 
 \EndIf
\EndFunction
}
\end{algorithmic}

\label{algo:filteringpD}
\end{algorithm}
 
Given $\delta \in (0.5,1)$, we consider an oracle subset $G^0$ of $S$ such that for a sufficiently small constant $c > 0$, 
\begin{equation}
    \frac{n - n_{G^0}}{n} +  \frac{\log(1/\delta)}{n} \leq c.
\end{equation}
The following theorem characterizes the performance of Algorithm~\ref{algo:filteringpD}.
\begin{theorem}\label{thm:filt_algo_pD}
Algorithm~\ref{algo:filteringpD} when instantiated with knowledge of $\norm{\Sigma_{G^0}}{2}$ stops in atmost $O\paren{(n - n_{G^0}) + \log(1/\delta)}$ steps and returns an estimate $\eparam_{\delta}$ such that with probability at least $1- \delta$,
\[ \mednorm{\eparam_{\delta} - \frac{1}{n_{G^0}} \sum_{x_i \in G^0} x_i}{2} \lesssim  \norm{\Sigma_{G^0}}{2}^{\half} \Big(\frac{n - n_{G^0}}{n} + \frac{\log(1/\delta)}{n}\Big)^{\half}  \]
\end{theorem}
The above result shows that if we are provided with additional side-information about the covariance of a large enough subset of points, then with high-probability it is algorithmically possible to come close to the mean of that subset. From our analysis in Section~\ref{sec:oracle}, we could view both robustness models via the lens of an oracle mixture model. Thus remarkably, the above result entails that the same algorithm provides a robust mean estimator for both models of robustness, given an upper bound on the spectral norm of the population covariance of the good component. For heavy tailed distributions, the good component is implicitly specified via an $\ell_2$ oracle, while for Huber contamination, the good component is the true uncontaminated distribution, but with an additional implicit $\ell_2$ oracle subset thereof. 

Hence, in order to instantiate the above theorem for the $\ell_2$ oracles $\calO(x) = \indic{\norm{x - \mu}{2} \leq R}$ discussed in Section~\ref{sec:oracle}, we need to bound the operator norm of covariance of the samples that belong to $P_\calO$:
\[ \widehat{\Sigma}^{\calO}_n = \paren{\sum \limits_{i=1}^n \calO(x_i)}^{-1} \sum\limits_{i=1}^n (x_i - \widehat{\mu}_n)(x_i - \widehat{\mu}_n)^T \calO(x_i) \]

As before, given $\delta \in (0.5,1)$, we consider values of $R$ and $n$ such that for a sufficiently small constant $c > 0$,
\begin{align}
\label{eqn:random2}
\Big(\frac{\sqrt{\trace{\Sigma}}}{R}\Big)^{2k} + \frac{\log(1/\delta)}{n} \leq c,
\end{align}
The following theorem characterizes the operator norm of $\widehat{\Sigma}^{\calO}_n$.

\begin{theorem}\label{thm:oracle_cov}
Let $P$ be any distribution with mean $\mu$ and covariance $\Sigma$ and bounded $2k$-moments for $k \in \{1,2\}$. 
For any $\delta \in (0.5,1)$ with probability at least $1 - \delta$,
\begin{align*}
    \norm{\widehat{\Sigma}^{\calO}_n}{2} \lesssim  \norm{\Sigma}{2} + R \norm{\Sigma}{2}^{\half} \sqrt{\frac{\log(p/\delta)}{n}} +  \frac{R^2\log(p/\delta)}{n}.
\end{align*}
\end{theorem}
We now have all the tools needed study the performance of Algorithm~\ref{algo:filteringpD} for both models of robustness. In particular, we combine Theorems~\ref{thm:filt_algo_pD} and~\ref{thm:oracle_cov} to characterize how well Algorithm~\ref{algo:filteringpD} can approximate the mean of the oracle set. When combined with Theorem~\ref{thm:oracelArgument}, we can then derive results on how well Algorithm~\ref{algo:filteringpD} can approximate the true mean. We follow this strategy and derive tight non-asymptotic results next.

\subsection{Heavy-Tailed Estimation}

We present our first result for heavy-tailed mean estimation for the distributions with bounded 4-moments. Given $\delta \in (0.5,1)$, we consider values of $n$ such that there is a small constant $c > 0$
\[{r^2(\Sigma) \frac{\log^2(p/\delta)}{n \log(1/\delta)}} \leq C  .\] Then, we have the following corollary.
\begin{corollary}\label{lem:heavy_mean4}
    Suppose $P$ has bounded 4th moment. Then, Algorithm~\ref{algo:filteringpD} when instantiated with $C\norm{\Sigma}{2}$ on $n$-\iid~samples from $P$ returns an estimate $\eparam_\delta$ such that, with probability at least $1 - \delta$,
\begin{align*}
\norm{\eparam_\delta - \mu}{2} \lesssim \opt_{n,\Sigma,\delta}
\end{align*}
\end{corollary}
{\bf Remark: } If $n$ is large enough so that for a small constant $c > 0$ the effective rank is bounded as $r^2(\Sigma) \leq c n \frac{\log(1/\delta)}{\log^2(p/\delta)} $, then Algorithm~\ref{algo:filteringpD} achieves the \emph{the optimal sub-Gaussian deviation bound.} The above presented result shows that it is possible to \emph{prune} samples to get high-probability bounds for the heavy-tailed problem. In comparison to the SDP based algorithms of~\citep{hopkins2018sub,cherapanamjeri2019fast}, our algorithm is easy to implement and practical. In particular, our estimator can also be computed in linear-time, requiring an overall runtime of $O(n p \log(1/\delta))$ compared to $O(n^4 + np)$ runtime of \citep{cherapanamjeri2019fast}.

Next, we show a somewhat weaker result for distributions with bounded 2nd moment.
\begin{corollary}\label{lem:heavy_mean2}
    Suppose $P$ has bounded 2nd moment. Then, Algorithm~\ref{algo:filteringpD} when instantiated with $C\norm{\Sigma}{2} + {\trace{\Sigma}\frac{\log(p/\delta)}{\log(1/\delta)}}$ on $n$-\iid~samples from $P$ returns an estimate $\eparam_\delta$ such that, with probability at least $1 - \delta$,
\begin{align*}
\norm{\eparam_\delta - \mu}{2} \lesssim \sqrt{\frac{\trace{\Sigma}\log(p/\delta)}{n}}
\end{align*}
\end{corollary}
{\bf Remark: } In the univariate setting, Corollary~\ref{lem:heavy_mean2} shows that Algorithm~\ref{algo:filteringpD} achieves the optimal sub-Gaussian deviation bound. As discussed in the introduction, even for the univariate setting, Catoni's M-estimation~\citep{catoni2012challenging} and Median-of-Means~\citep{alon96,nemirovski1983problem,jerrum86} are the only known estimators to achieve these rates for any 2nd moment bounded distribution. Algorithm~\ref{algo:filteringpD} is the \emph{first sample-pruning} based estimator, which achieves optimal these optimal bounds, without any further assumptions.

In the multivariate setting, while our theoretical upper bounds are weaker than the guarantees of GMOM, we conduct extensive simulations in Section~\ref{sec:experiments} which suggest otherwise.

\subsection{Huber's $\epsilon$-contamination}
In this section, we present results for robust mean estimation in the $\epsilon$-contamination model. Recall that in this setting, we observe samples drawn from $P_\epsilon$, which for an arbitrary distribution $Q$ is defined as a mixture model, $P_{\epsilon} = (1-\epsilon)P + \epsilon Q$.

When we observe $n$ samples from $P_{\epsilon}$, we know that there will roughly be $n(1-\epsilon)$ points drawn from the true distribution $P$. This implies that we can again rely on the presence of $\ell_2$ oracles, $\calO(x) = \indic{\norm{x - \mu}{2} \leq R}$. We first present results for distributions with bounded 2nd moment. 

Given $\delta \in (0.5,1)$, we consider values of $\epsilon$ and $n$ such that for a sufficiently small constant $c > 0$,
\begin{align}
\label{eqn:random}
\epsilon + \sqrt{\epsilon \frac{\log(1/\delta)}{n}} + \frac{\log(1/\delta)}{n} \leq c,
\end{align}

\begin{corollary}\label{lem:filt_algo_pD_2}
Suppose $P$ has bounded 2nd moments. Then, given $n$~\iid samples from the mixture distribution~\eqref{eqn:huber_mixture}, Algorithm~\ref{algo:filteringpD} when instantiated with $C \norm{\Sigma}{2} + \frac{\trace{\Sigma} \log(p/\delta)}{n \epsilon + \log(1/\delta)}$ returns an estimate $\eparam_\delta$ such that with probability at least $1 - \delta$, 
\begin{align*}
\norm{\eparam - \mu}{2} \lesssim \norm{\Sigma}{2}^{\half} \sqrt{\epsilon} + \sqrt{\frac{\trace{\Sigma}\log(p/\delta)}{n}}
\end{align*}
\end{corollary}
Algorithm~\ref{algo:filteringpD} achieves the optimal asymptotic bias $O(\norm{\Sigma}{2}^{\half} \sqrt{\epsilon})$~\citep{lai2016agnostic}, but in higher dimensions, we don't get the sub-gaussian deviation term. We next present results for distributions with bounded 4th moment. 

\begin{corollary}\label{lem:filt_algo_pD_4}
Suppose $P$ has bounded 4th moments. Then, given $n$~\iid samples from the mixture distribution~\eqref{eqn:huber_mixture}, Algorithm~\ref{algo:filteringpD} when instantiated with  $C\norm{\Sigma}{2} + \frac{\trace{\Sigma} \log(p/\delta)}{\sqrt{n^2 \epsilon + n \log(1/\delta)}}$ returns an estimate $\eparam_\delta$ such that with probability at least $1 - \delta$, 
\begin{align*}
\norm{\eparam - \mu}{2} \lesssim \norm{\Sigma}{2}^{1/2} \sqrt{\epsilon} + \opt_{n,\Sigma,\delta} + {\sqrt{\trace{\Sigma}}} \sqrt{\frac{\log(p/\delta)}{n_{G^0}}}\paren{\epsilon + \frac{\log(1/\delta)}{n}}^{1/4}
\end{align*}
\end{corollary}
In this case, Algorithm~\ref{algo:filteringpD} has both a sub-optimal asymptotic bias $O(\norm{\Sigma}{2}^{\half} \sqrt{\epsilon})$, compared to the optimal bias of $O(\norm{\Sigma}{2}^{\half} {\epsilon}^{3/4})$. Moreover, when $\epsilon = \Omega(1)$, it has a sub-optimal deviation as well.

\begin{table}[]
\centering 
\caption{Performance of Filtering Algorithm for Multivariate Mean Estimation}
\begin{tabular}{@{}lll@{}}
\toprule
           & Heavy-Tailed Model    & Huber's $\epsilon$-contamination($\epsilon = \Omega(1)$) \\ \midrule
2nd Moment & Sub-optimal deviation & Optimal Asymptotic Bias          \\ \midrule
4th Moment & Optimal Deviation     & Sub-optimal Asymptotic Bias      \\ \bottomrule
\end{tabular}
\label{tab:perf}
\vspace{-10pt}
\end{table}

\paragraph{Discussion.} 
Our upper bounds for the filtering estimator are summarized in Table~\ref{tab:perf}. Our results suggest a dichotomy for multivariate mean estimation, the settings in which Filtering achieves optimal deviation in the heavy-tailed model, it fails to achieve the correct asymptotic bias in the $\epsilon$-contamination setting. Hence, a natural question to ask is whether there exist estimators, which achieve \emph{both} optimal asymptotic bias in the $\epsilon$-contamination model, and simultaneously achieve the optimal deviation in the heavy-tailed model. We study this question next.

\section{Optimal Asymptotic Bias and (near)-Optimal Deviation.}\label{sec:app_deviation}

In this section, we study candidate estimators which achieve \emph{both} optimal asymptotic bias in the $\epsilon$-contamination model, and simultaneously achieve the optimal deviation in the heavy-tailed model. To begin with, we study two candidate estimators in the $\epsilon$-contamination model.

\subsection{Some Candidate Estimators}

 \textbf{Convex M-estimation.} M-estimators were originally proposed by Huber~\citep{huber1965robust}, and were shown to be robust in one dimension. Subsequent research in 1970s showed that in multivariate data, M-estimators performed poorly~\citep{maronna1976robust}. In particular,~\citet{donoho1992breakdown}~showed that when the data is $p$-dimensional, the breakdown point of M-estimators scales inversely with the dimension.~\citet{lai2016agnostic} and~\citet{diakonikolas2016robust} derived negative results for the geometric median. We further extend this observation, and show that even at a very small contamination level,\ie $\epsilon \mapsto 0$, the bias of any convex M-estimator which is Fisher-consistent for $\calN(0,\calI_p)$ will necessarily scale polynomially in the dimension.

\begin{lemma}\label{lem:convex_huber}
Let $P = \calN(0,\calI_p)$ and consider the convex risk $R_P(\theta) = \Exp_{z \sim P} [\ell(\norm{z - \theta}{2})]$ where $\ell : \real \mapsto \real$ be any twice differentiable Fisher-consistent convex loss, \ie $\theta(P) = \argmin_\theta R_P(\theta) = 0$. Then, there exists a corruption $Q$ such that $\lim \limits_{\epsilon \mapsto 0} \norm{\theta(P_\epsilon)}{2} \geq \epsilon \sqrt{p}$
\end{lemma}

\textbf{Subset Search.} Having ruled out convex estimation to a certain extent, we next turn our attention to non-convex methods. Perhaps the most simple non-convex method is simple search. Intuitively, the squared loss measures the \emph{fit} between a parameter $\theta$ and samples $\calZ$, and if all samples don't come from the same distribution(\ie have outliers), then the corresponding \emph{fit} should be bad. To capture this intuition algorithmically, one can consider all subsets of size $\floor{(1 - \epsilon)n}$, minimize the squared loss over these subsets, and then return the estimator corresponding to the subset with least squared loss or best fit. To be precise, given $n$ samples from $P_{\epsilon}$ 
\begin{align}
 S^* & \defeq \argmin \limits_{S~\st~ |S| = (1 - \epsilon)n} \min \limits_{\theta} \frac{1}{(1 - \epsilon)n} \sum \limits_{x_i \in S} \norm{x_i - \theta}{2}^2 \nonumber \\
\eSRM &\defeq  \min \limits_{\theta} \frac{1}{(1 - \epsilon)n} \sum \limits_{x_i \in S^*} \norm{x_i - \theta}{2}^2
\end{align}
Our next result studies the asymptotic performance of this estimator. 
\begin{lemma}\label{lem:srm_mean}
Let $P = \calN(0,\calI_p)$, then as $n \mapsto \infty$, we have that
\begin{align} \label{eqn:thm:srm_mean2}
\sup \limits_Q \norm{\eSRM - \Exp_{x \sim P}[x]}{2} = \frac{\epsilon}{\sqrt{(1 - \epsilon)(1 - 2 \epsilon)}} \sqrt{\trace{\Sigma(P)}}.
\end{align}
\end{lemma}
The above result shows that the bias of this estimator necessarily scales with the dimension.

\subsection{Optimal Univariate Mean Estimators}

In the previous section, we studied candidate estimators and showed that they don't achieve the optimal asymptotic bias in $\epsilon$-contamination model for multivariate mean estimation. In this section, we take a step back, and study univariate mean estimation.

We study the interval estimator which was initially proposed by \citep{lai2016agnostic}. The estimator, presented in Algorithm~\ref{algo:interval1D}, proceeds by using half of the samples to identify the shortest interval containing at least $(1 - \epsilon)n$ fraction of the points, and then the remaining half of the points is used to return an estimate of the mean. 
 \begin{algorithm}[H]
\centering
\caption{Robust Univariate Mean Estimation}
        \begin{algorithmic}
\Function{Interval1D($\{z_i\}_{i=1}^{2n}$,Corruption Level $\epsilon$, Confidence Level $\delta$)}{}
\State Split the data into two subsets: $\calZ_1 = \{z_i\}_{i=1}^{n}$ and $\calZ_2 = \{z_i\}_{i=n+1}^{2n}$. 
\State Let $\alpha = \max \paren{\epsilon,\frac{\log (1/\delta)}{n}}$.
\State Using $\calZ_1$, let $\hat{I} = [a,b]$ be the shortest interval containing $n \paren{ 1-2\alpha - \sqrt{2\alpha \frac{\log(4/\delta)}{n}} - \frac{\log(4/\delta)}{n}}$ points.
\State Use $\calZ_2$ to identify points lying in $[a,b]$.
\State \Return $\frac{1}{\sum_{i=n}^{2n}\indic{z_i \in \hat{I}}} \sum_{i=n}^{2n} z_i \indic{z_i \in \hat{I}} $
\EndFunction
\end{algorithmic}
\label{algo:interval1D}
\end{algorithm}
We assume that the contamination level $\epsilon$ and confidence level $\delta$ are such that,
\[ 2\epsilon + \sqrt{\epsilon \frac{\log(4/\delta)}{n}} + \frac{\log(4/\delta)}{n} < \half. \]
Then, we have the following Lemma.
\begin{lemma}\label{lem:interval_p1D_huber}
Suppose $P$ be any $2k$-moment bounded distribution over $\real$ with mean $\mu$ with variance bounded by $\sigma^2$. Given, $n$ samples $ \{x_i\}_{i=1}^n$ from the mixture distribution~\eqref{eqn:huber_mixture}, Algorithm~\ref{algo:interval1D} returns an estimate $\eparam_\delta$ such that with probability at least $1-\delta$,
\[ | \eparam_\delta - \mu | \lesssim  \sigma\max\paren{2\epsilon,\frac{\log(1/\delta)}{n}}^{1 - \frac{1}{2k}} +   \sigma \paren{\frac{\log n}{n}}^{1 - \frac{1}{2k}} +  \sigma \sqrt{\frac{\log (1/\delta)}{n}} \]
\end{lemma}

Observe that~Algorithm~\ref{algo:interval1D} has an asymptotic bias of $O(\sigma \epsilon^{1 - 1/2k})$ in the $\epsilon$-contamination setting, which is known to be information theoretically optimal~\citep{hopkins2018mixture,lai2016agnostic}. 

  Moreover, in the heavy-tailed model when $P$ has atleast bounded 4th moment, \ie $k \geq 2$, $\frac{\log(n)}{n}^{1 - 1/2k}$ term can be ignored for large enough $n$. Hence, for $k \geq 2$ and large enough $n$, Algorithm~\ref{algo:interval1D} achieves the desired bound of $O\paren{\sigma \epsilon^{1 - 1/2k} + \opt_{n,\sigma^2,\delta}}$ and thus has both asymptotic optimal asymptotic bias and optimal sub-gaussian deviation. 
  
  \textbf{Remark.} When $P$ has bounded second moment, \ie $k=1$, Corollary~\ref{lem:filt_algo_pD_2} shows that Algorithm~\ref{algo:filteringpD} achieves $O\paren{\sigma \sqrt{\epsilon} + \opt_{n,\sigma^2,\delta}}$ error rate, and hence, has optimal asymptotic bias and optimal sub-gaussian deviation.

This result shows that at least in the univariate setting, there do exist estimators which are simultaneously optimal in both models of robustness.

\subsection{(near)-Optimal Multivariate Mean Estimators.}

In this section, we show how to extend our univariate estimators to the multivariate setting, and hence derive estimators which are simultaneously optimal(or near-optimal) for both models of robustness. 

For ease of notation, let $\text{INTERVAL1D}$ be the robust univariate estimator from the previous section. Then, we use it construct the following multivariate estimator, $\eparam$ which takes in $n$-samples $\{x_i\}_{i=1}^n$:
\begin{align}\label{eqn:multiEst}
    \eparam(\{x_i\}_{i=1}^n) = \inf \limits_{\theta} \sup \limits_{u \in \calN^{1/2}(\calS^{p-1})} | u^T\theta - \text{INTERVAL1D}(\{u^T x_i\}_{i=1}^n,\epsilon,\frac{\delta}{5^p}) |,
\end{align}
where $\calN^{1/2}(\calS^{p-1})$ is the half-cover of the unit sphere $\calS^{p-1}$, \ie $\forall u \in \calS^{p-1}$, there exists a $y \in \calN^{1/2}(\calS^{p-1})$ such that $u = y + z$ for some  $\norm{z}{2} \leq \half$.

The proposed estimator proceeds by robustly estimating the mean along almost every direction $u$, and returns an estimate $\eparam$, whose projection along $u$($u^T \eparam$) is close to these robust estimates. Such directional-control based estimators have been previously studied in the context of heavy-tailed mean estimation by \citep{joly2017estimation} and~\citep{catoni2017dimension}.~\citet{joly2017estimation} used the median-of-means estimator, while~\citet{catoni2017dimension} used Catoni's M-estimator~\citep{catoni2012challenging} as their univariate estimator. 

In the multivariate setting, we further assume that the contamination level $\epsilon$, and confidence are such that,
\[ 2\epsilon + \sqrt{\epsilon\paren{\frac{p}{n} + \frac{\log(1/\delta)}{n}}} + \frac{p}{n} + \frac{\log(4/\delta)}{n} < c, \]
for some small constant $c>0$. Then, we have the following result.
\begin{lemma}\label{lem:ppEst}
Suppose $P$ has bounded $2k$ moments with mean $\mu$ and covariance $\Sigma$. Given $n$ samples $ \{x_i\}_{i=1}^n$ from the mixture distribution~\eqref{eqn:huber_mixture}, we get that with probability at least $1 - \delta$,
\begin{align} 
\norm{\eparam(\{x_i\}_{i=1}^n) - \mu }{2}  \lesssim \norm{\Sigma}{2}^{1/2} \epsilon^{1 - 1/{2k}} + \norm{\Sigma}{2}^{1/2}{ \sqrt{\frac{\log (1/\delta)}{n}} +  \norm{\Sigma}{2}^{1/2} \sqrt{\frac{p}{n}}} + \norm{\Sigma}{2}^{1/2}\paren{\frac{\log n}{n}}^{1 - \frac{1}{2k}}
\end{align}
\end{lemma}

Observe that the estimator proposed in \eqref{eqn:multiEst} achieves a dimension independent asymptotic bias of $O(\sigma \epsilon^{1 - 1/2k})$ in the $\epsilon$-contamination model for multivariate mean estimation. 

  Moreover, in the heavy-tailed model when $P$ has at least bounded 4th moment, \ie $k \geq 2$, $\frac{\log(n)}{n}^{1 - 1/2k}$ term can be ignored for large enough $n$. Hence, for $k \geq 2$ and large enough $n$, the proposed estimator achieves a deviation bound of $O\paren{\norm{\Sigma}{2}^{1/2}{ \sqrt{\frac{\log (1/\delta)}{n}} +  \norm{\Sigma}{2}^{1/2} \sqrt{\frac{p}{n}}}}$. Recall that the optimal subgaussian deviation bound is given by
  \[ \opt_{n,\Sigma,\delta} = \sqrt{\frac{\trace{\Sigma}}{n}} + \sqrt{\frac{\norm{\Sigma}{2}\log(1/\delta)}{n}} \]
We see that the proposed estimator achieves a nearly-optimal deviation bound. In particular, for any covariance matrix $\Sigma$, since $\trace{\Sigma} \leq \norm{\Sigma}{2} p$, hence, the proposed estimator is near-optimal. However, for nearly-spherical distributions, \ie distributions for which $\trace{\Sigma} \approx \norm{\Sigma}{2} p$, the estimator has the optimal deviation $\opt_{n,\Sigma,\delta}$. The above result also shows that there do exist estimators, which achieve both achieve both optimal asymptotic bias in the $\epsilon$-contamination model and (nearly)-optimal deviation in the heavy-tailed model. For distributions with only bounded second moment, \ie $k=1$, the estimator proposed in~\eqref{eqn:multiEst} achieves a sub-optimal deviation bound. 

However, one can define a multivariate estimator similar to~\eqref{eqn:multiEst}, where we replace the interval estimator with the univariate version of Algorithm~\ref{algo:filteringpD}. In particular, consider the following estimator
\begin{align}\label{eqn:multiEst_univ}
    \eparam_{\text{Filt}}(\{x_i\}_{i=1}^n) = \inf \limits_{\theta} \sup \limits_{u \in \calN^{1/2}(\calS^{p-1})} | u^T\theta - \text{FILTER1D}(\{u^T x_i\}_{i=1}^n,\epsilon,\frac{\delta}{5^p}) |,
\end{align} 
where \text{FILTER1D($\cdot$)} is the univariate version of Algorithm~\ref{algo:filteringpD}, and $\calN^{1/2}(\calS^{p-1})$ is the half-cover of the unit sphere. Then, we have the following result.

\begin{corollary}~\label{cor:filt_multi}
Suppose $P$ has mean $\mu$ and covariance $\Sigma$. Given $n$ samples $ \{x_i\}_{i=1}^n$ from the mixture distribution~\eqref{eqn:huber_mixture}, we get that with probability at least $1 - \delta$,
\begin{align} 
\norm{\eparam_{\text{Filt}}(\{x_i\}_{i=1}^n) - \mu }{2}  \lesssim \norm{\Sigma}{2}^{1/2} \sqrt{\epsilon} + \norm{\Sigma}{2}^{1/2}{ \sqrt{\frac{\log (1/\delta)}{n}} +  \norm{\Sigma}{2}^{1/2} \sqrt{\frac{p}{n}}}
\end{align}
\end{corollary}

Note that our multivariate mean estimators essentially rely only robust univariate estimation. \citet{prasad2018robust} and ~\citet{diakonikolas2018sever} showed that by using a robust multivariate mean estimator to estimate gradients robustly, one can do robust risk minimization nearly-optimally for a broad class of models such as linear regression, generalized linear models and exponential families. Our results show that at least statistically, if we can solve robust univariate mean estimation optimally, then, we can solve robust risk minimization optimally for a broad class of problems.

\noindent Next, we extend our proposed estimator for sparse mean estimation.

\paragraph{Sparse Mean Estimation.} In this setting, we further assume that the true mean vector of the distribution $P$ has only a few non-zero co-ordinates, \ie it is sparse. Such sparsity patterns are known to be present in high-dimensional data(see \citep{rish2014practical} and references therein). Then, the goal is to design estimators which can exploit this sparsity structure, while remaining robust under both the $\epsilon$-contamination model and the heavy-tailed model. Formally, for a vector $x \in \real^p$, let $\text{supp}(x) = \{ i \in [p]~~\st~~x(i) \neq 0 \}$. Then, $x$ is $s$-sparse if $|\text{supp}(x)| \leq s$. We further assume that $s \leq p/2$. Let $\Theta_s$ be the set of s-sparse vectors in $\real^p$, and let $\calN^{\half}_{2s}(\calS^{p-1})$ is the half-cover of the set of unit vectors which are 2s-sparse. Then, in this setting, we propose the following estimator:

\begin{align}\label{eqn:multiEst_sparse}
    \eparam_s(\{x_i\}_{i=1}^n) = \inf \limits_{\theta \in \Theta_s} \sup \limits_{u \in \calN^{1/2}_{2s}(\calS^{p-1})} | u^T\theta - \text{INTERVAL1D}(\{u^T x_i\}_{i=1}^n,\epsilon,\frac{\delta}{(\frac{6ep}{s})^s}) |,
\end{align}
We further assume that the contamination level $\epsilon$, and confidence are such that,
\[ 2\epsilon + \sqrt{\epsilon\paren{\frac{s \log p}{n} + \frac{\log(1/\delta)}{n}}} + \frac{s\log p}{n} + \frac{\log(4/\delta)}{n} < c, \]
for some small constant $c>0$. Then, we have the following result.
\begin{corollary}\label{cor:ppEst_sparse}
Suppose $P$ has bounded $2k$ moments with mean $\mu$ and covariance $\Sigma$, where $\mu$ is $s$-sparse. Then, given $n$ samples $ \{x_i\}_{i=1}^n$ from the mixture distribution~\eqref{eqn:huber_mixture}, we get that with probability at least $1 - \delta$,
\begin{align} 
\norm{\eparam_s(\{x_i\}_{i=1}^n) - \mu }{2}  \lesssim \norm{\Sigma}{2,2s}^{1/2} \epsilon^{1 - 1/{2k}} + \norm{\Sigma}{2,2s}^{1/2}{ \sqrt{\frac{\log (1/\delta)}{n}} +  \norm{\Sigma}{2,2s}^{1/2} \sqrt{\frac{s \log p}{n}}} + \norm{\Sigma}{2,2s}^{1/2}\paren{\frac{\log n}{n}}^{1 - \frac{1}{2k}},
\end{align}
where $\norm{\Sigma}{2,2s} = \sup_{u \in \calS^{p-1}, \norm{u}{0} \leq 2s} u^T \Sigma u$.
\end{corollary}
The above result shows that the proposed estimator exploits the underlying sparsity structure, and achieves the near-optimal deviation rate of $O\paren{\norm{\Sigma}{2,2s}^{1/2} \sqrt{\frac{s \log p}{n}}}$, while simultaneously achieving the optimal asymptotic bias of $O(\norm{\Sigma}{2,2s}^{1/2} \epsilon^{1 - 1/{2k}})$.

As stated before, for the special case of $k=1$, \ie when the distribution $P$ has only bounded second moments, we can use the univariate version of Algorithm~\ref{algo:filteringpD} to construct a similar estimator for sparse mean estimation. In particular, consider the following estimator
\begin{align}\label{eqn:multiEst_sparse_filt}
    \eparam_{s,\text{Filt}}(\{x_i\}_{i=1}^n) = \inf \limits_{\theta \in \Theta_s} \sup \limits_{u \in \calN^{1/2}_{2s}(\calS^{p-1})} | u^T\theta - \text{FILTER1D}(\{u^T x_i\}_{i=1}^n,\epsilon,\frac{\delta}{(\frac{6ep}{s})^s}) |.
\end{align}
Then, we have the following result.

\begin{corollary}\label{cor:ppEst_sparse}
Suppose $P$ has mean $\mu$ and covariance $\Sigma$, where $\mu$ is $s$-sparse. Then, given $n$ samples $ \{x_i\}_{i=1}^n$ from the mixture distribution~\eqref{eqn:huber_mixture}, we get that with probability at least $1 - \delta$,
\begin{align} 
\norm{\eparam_{s,\text{Filt}}(\{x_i\}_{i=1}^n) - \mu }{2}  \lesssim \norm{\Sigma}{2,2s}^{1/2} \sqrt{\epsilon} + \norm{\Sigma}{2,2s}^{1/2}{ \sqrt{\frac{\log (1/\delta)}{n}} +  \norm{\Sigma}{2,2s}^{1/2} \sqrt{\frac{s \log p}{n}}}
\end{align}
where $\norm{\Sigma}{2,2s} = \sup_{u \in \calS^{p-1}, \norm{u}{0} \leq 2s} u^T \Sigma u$.
\end{corollary}

\section{Experiments}\label{sec:experiments}
In this section, we conduct synthetic experiments to study the performance of our proposed estimators for heavy-tailed mean estimation.

\begin{figure}[t]
        \centering
         \subfigure[\small{$CDF[\norm{\eparam - \tparam}{2}]$} \label{fig:log1}]{\includegraphics[width=0.235\textwidth]{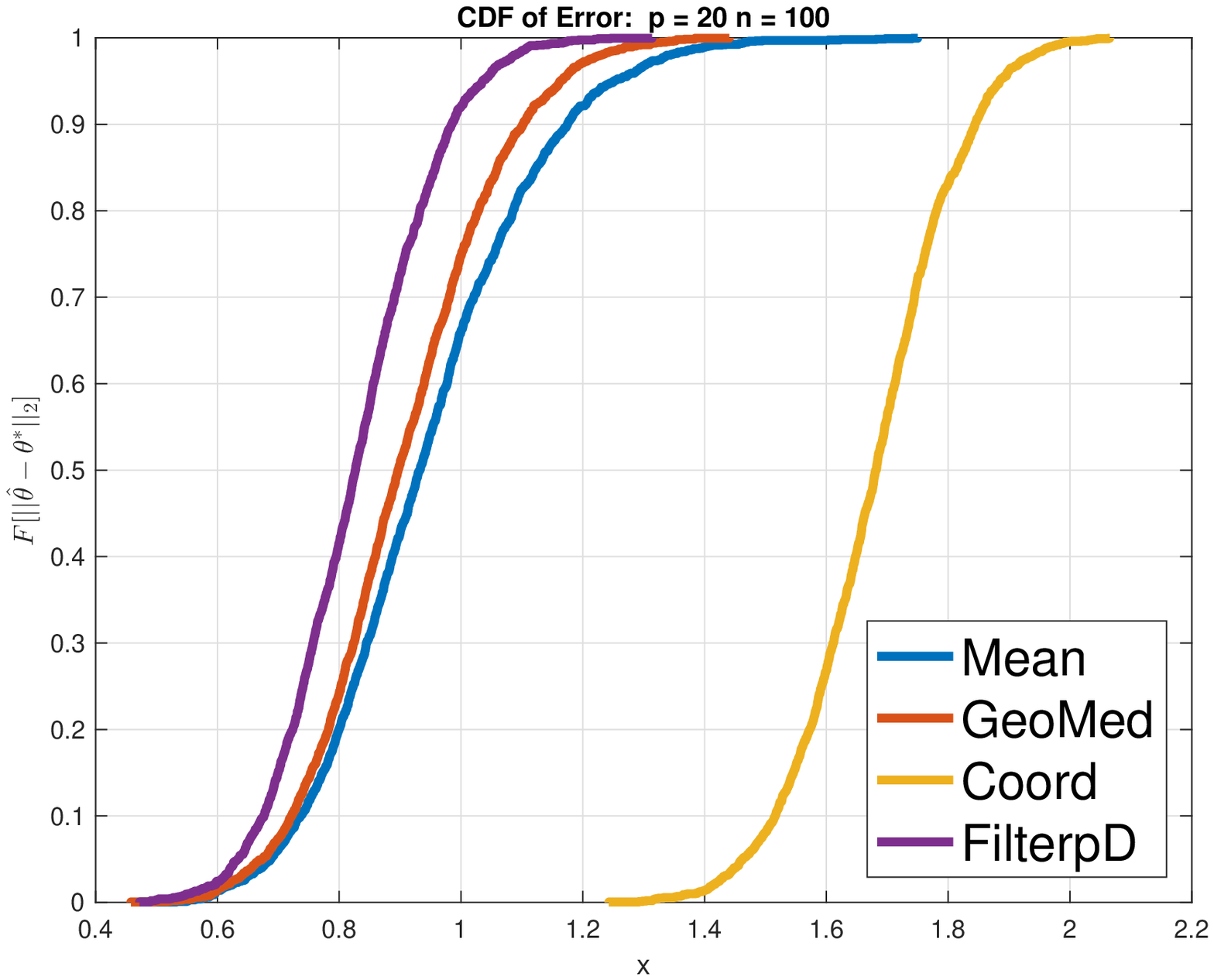}}
        \subfigure[ \small{$\log(Q_{\delta}(\ell_{\eparam}))$ vs $\delta$} \label{fig:log2}]{\includegraphics[width=0.235\textwidth]{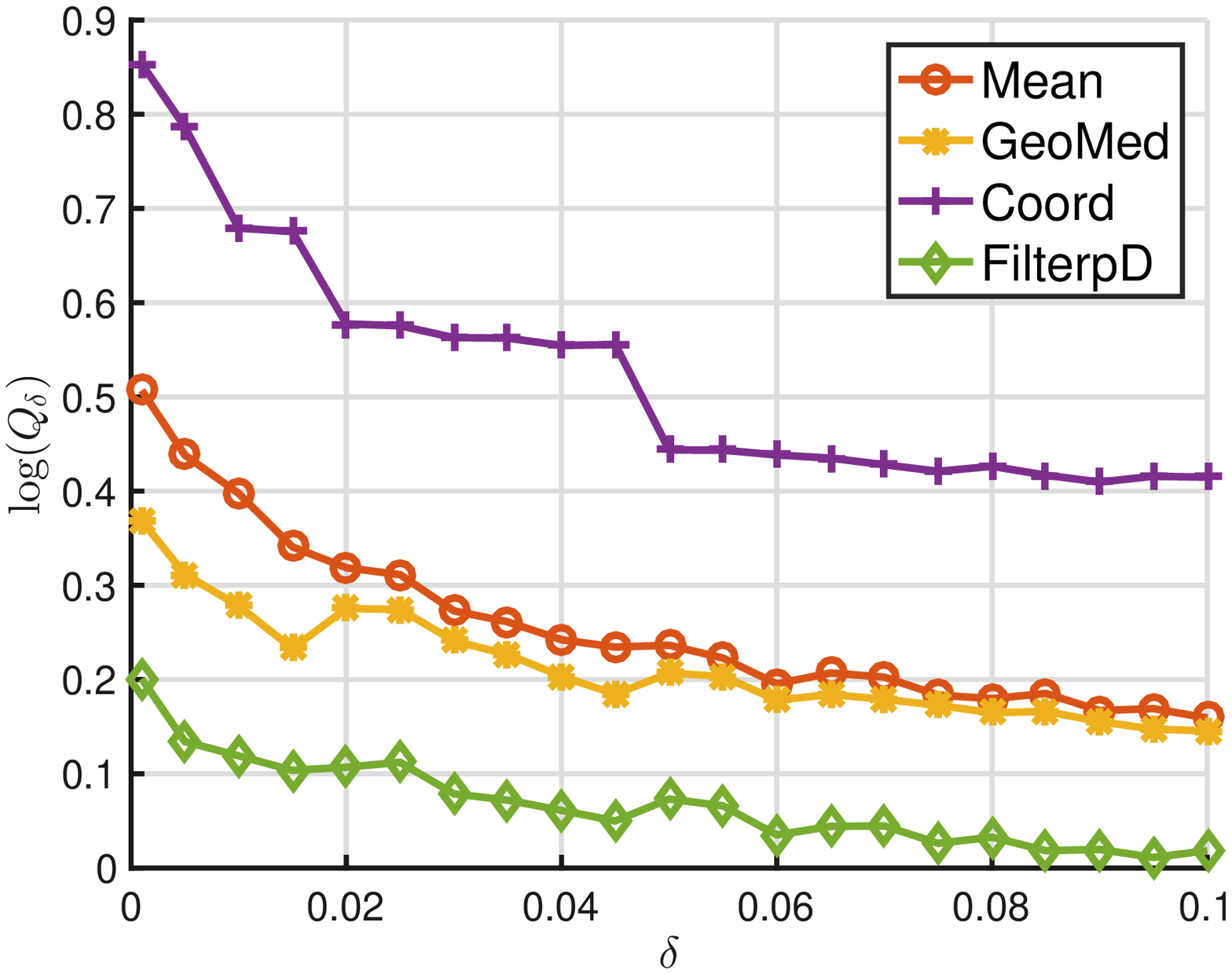}}
        \subfigure[ \small{$(Q_{\delta}(\ell_{\eparam}))$ vs $n$} \label{fig:log3}] {\includegraphics[width=0.235\textwidth]{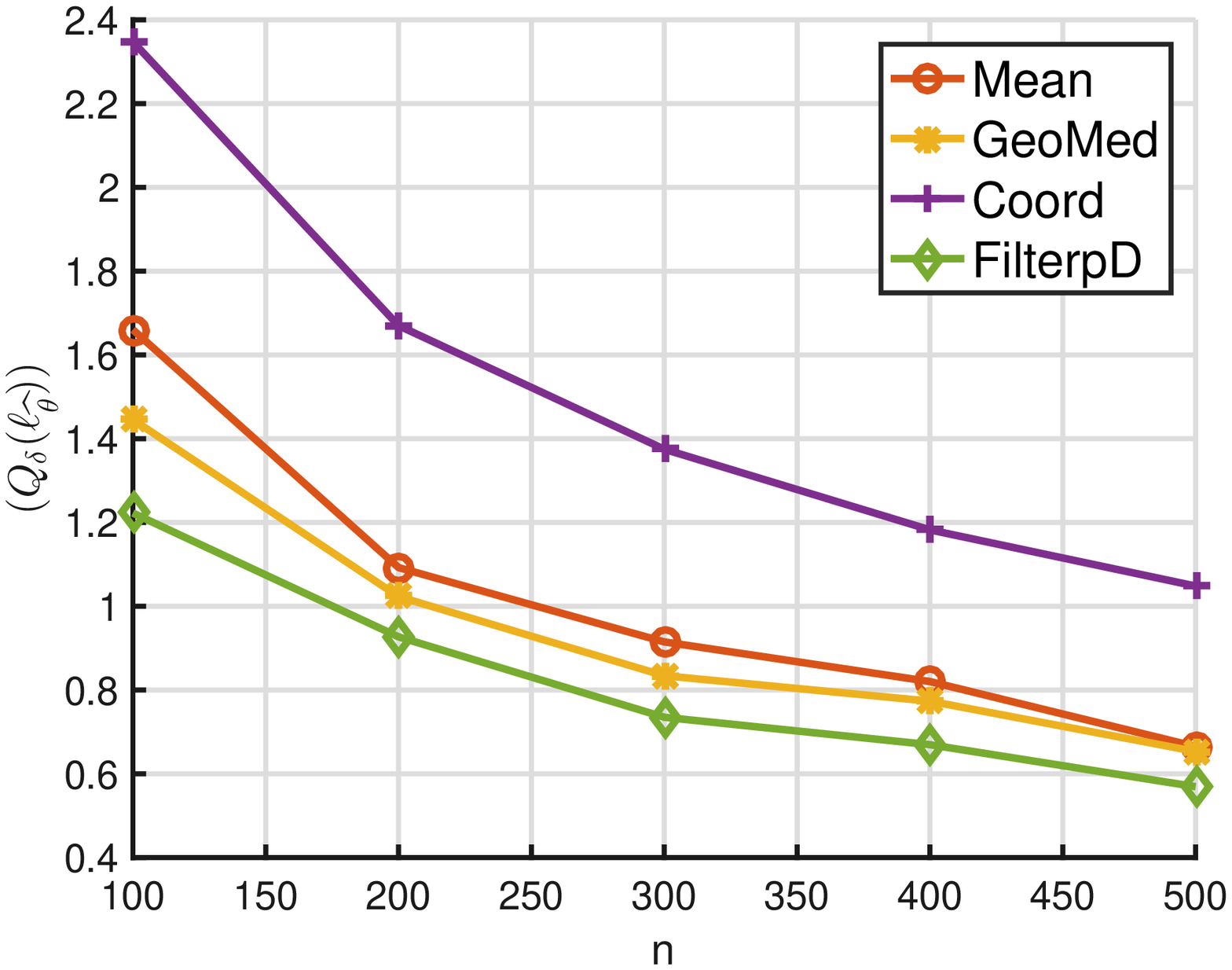}}
        \subfigure[\small{$(Q_{\delta}(\ell_{\eparam}))$ vs $p$}
        \label{fig:log4}]{\includegraphics[width=0.235\textwidth]{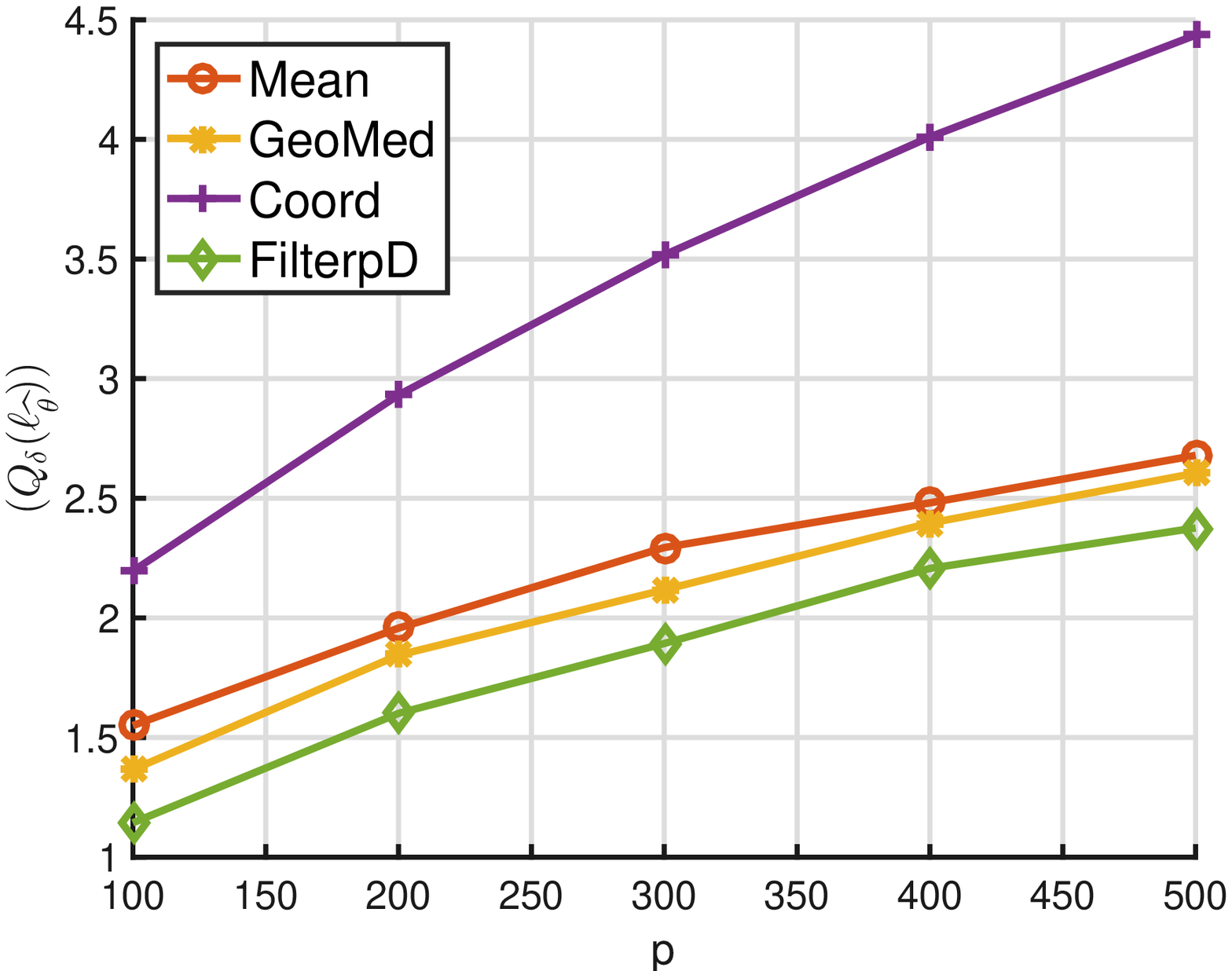}}
        \vspace{-10pt}
\caption{Mean Estimation for Multivariate LogNormal Noise}
\label{fig:log}
\end{figure}

\paragraph{Setup.} We generate $x \in \real^p$ from an isotropic zero-mean heavy-tailed distribution. We experiment with two different distributions: (1) Log-normal distribution and (2) Pareto Distribution. For Pareto-distribution with tail-parameter $\beta$, the $k^{th}$ order moments exists only if $k < \beta$, hence, smaller the $\beta$, the more heavy-tailed the distribution. We fix $k=3$. In this setup, we experiment with different $n,p$ and $\delta$. For each setting of $(n,p,\delta)$, cumulative metrics are reported over $2000$ trials. We vary $n$ from 100 to 500, and $p$ from 20 to 100.
\paragraph{Methods.} We compare the filtering estimator with several baselines: (1) Sample mean, (2) Geometric Median of Means~\citep{minsker2015geometric} which we refer to as GeoMed, and (3) Co-ordinate wise Robust Estimation, where we apply the univariate version of filtering on each co-ordinate independently, which we refer to as Coord.
\paragraph{Metric.} For any estimator($\eparam_{n,\delta}$), we use $\ell(\eparam_{n,\delta}) = \norm{\eparam - \mu(P)}{2}$ as our primary metric. For each setting of $(n,p,\delta)$, we run the experiment for 2000 trials, which gives us access to the distribution of $\ell_{\eparam_{n,\delta}}$. Since, we care about the deviation performance, measure the quantile error of the estimator, \ie $Q_{\delta}(\eparam) = \inf\{\alpha:\,\Pr(\ell(\eparam) > \alpha) \leq \delta\}$. This can also be thought of as the length of confidence interval for a confidence level of $1 - \delta$. 
\paragraph{Hyperparameter Tuning.} Apart from sample mean, all other estimators take into knowledge of $\delta$, which is the desired confidence level. For GeoMed, the number of blocks $k$ is set to $\ceil{2 \log(1/\delta)}$. Similarly, for FilterpD, we run the filtering procedure for $\ceil{2 \log(1/\delta)}$ steps, \ie at each step we sample a point based on its projection score, and throw it away.

\paragraph{Results.} In Figures~\ref{fig:log} and~\ref{fig:pareto}, we see that our filtering estimator clearly outperforms all baselines across several metrics. Figures~\ref{fig:log1} and \ref{fig:pareto1} show that CDF of the loss of our estimator is clearly to the left(and hence) better. Figures~\ref{fig:log2} and \ref{fig:pareto2} show that for any confidence level $1-\delta$, the length of the oracle confidence interval ($Q_{\delta}(\eparam)$) for our estimator is better than all baselines. We also see better sample dependence in Figures~\ref{fig:log3} and \ref{fig:pareto3}, and better dimension dependence in Figures~\ref{fig:log4} and  \ref{fig:pareto4}.

\begin{figure}[tbp]
        \centering
         \subfigure[\small{$CDF[\norm{\eparam - \tparam}{2}]$} \label{fig:pareto1}]{\includegraphics[width=0.235\textwidth]{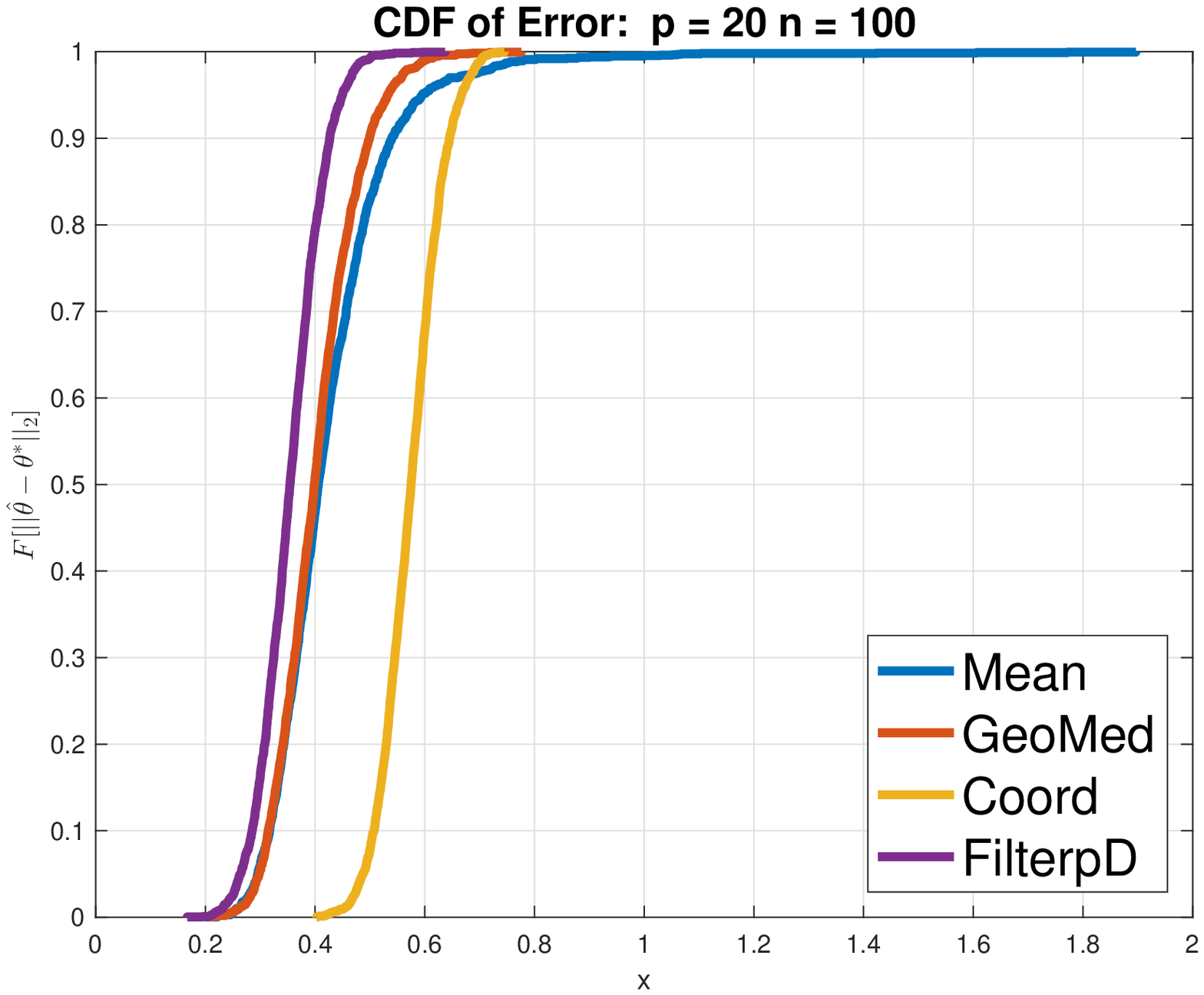}}
        \subfigure[ \small{$\log(Q_{\delta}(\ell_{\eparam}))$ vs $\delta$} \label{fig:pareto2}]{\includegraphics[width=0.235\textwidth]{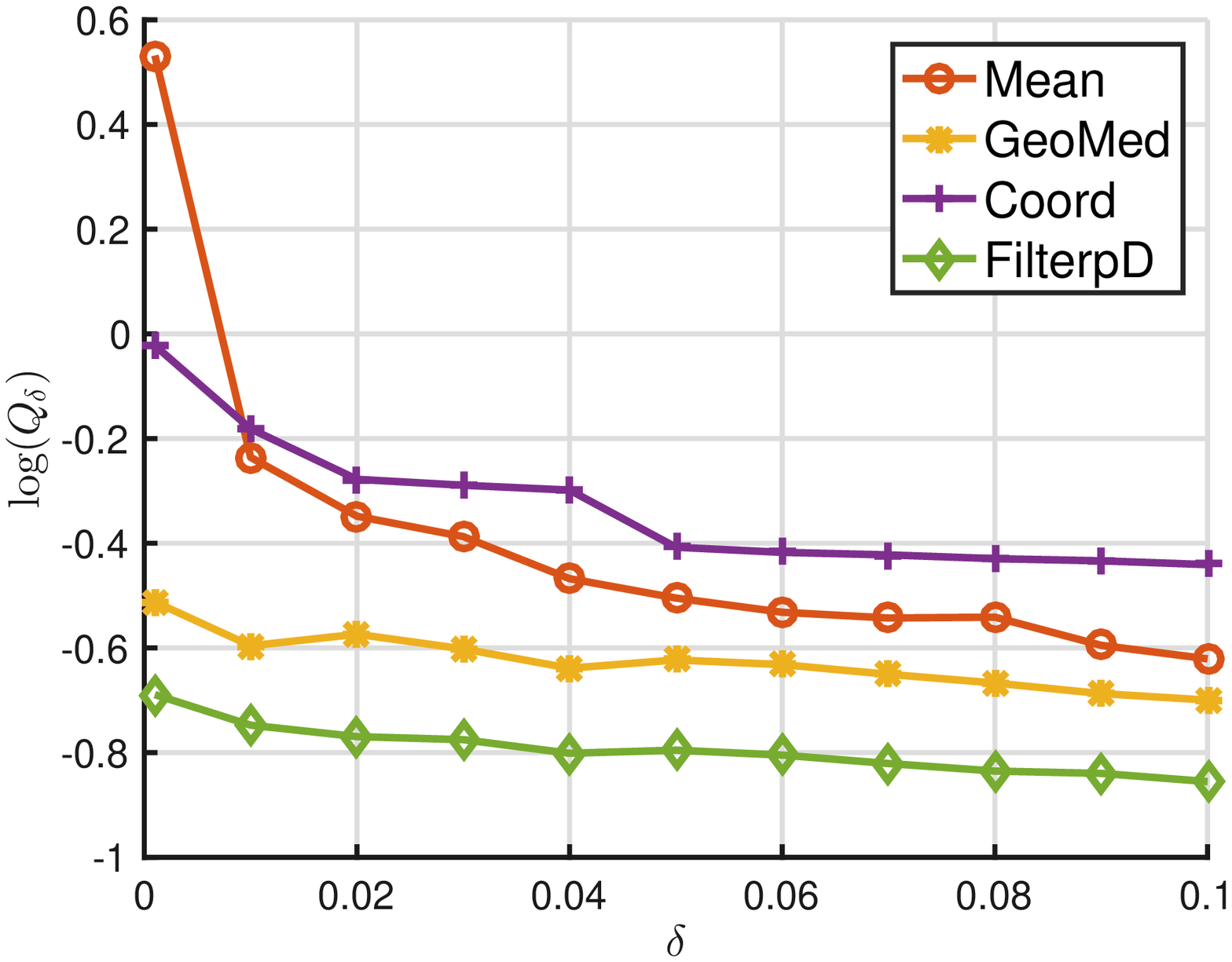}}
        \subfigure[ \small{$(Q_{\delta}(\ell_{\eparam}))$ vs $n$} \label{fig:pareto3}] {\includegraphics[width=0.235\textwidth]{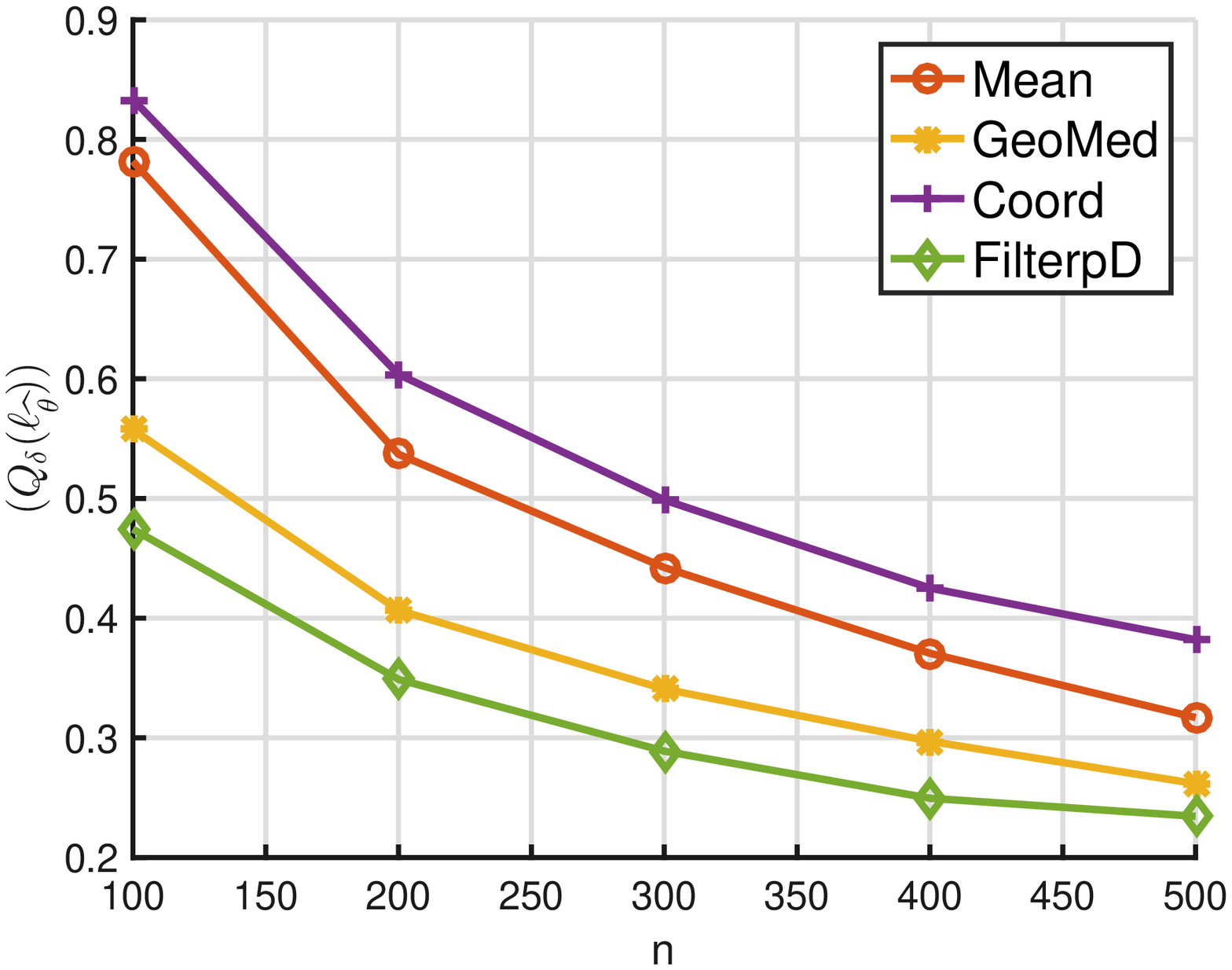}}
        \subfigure[\small{$(Q_{\delta}(\ell_{\eparam}))$ vs $p$}
        \label{fig:pareto4}]{\includegraphics[width=0.235\textwidth]{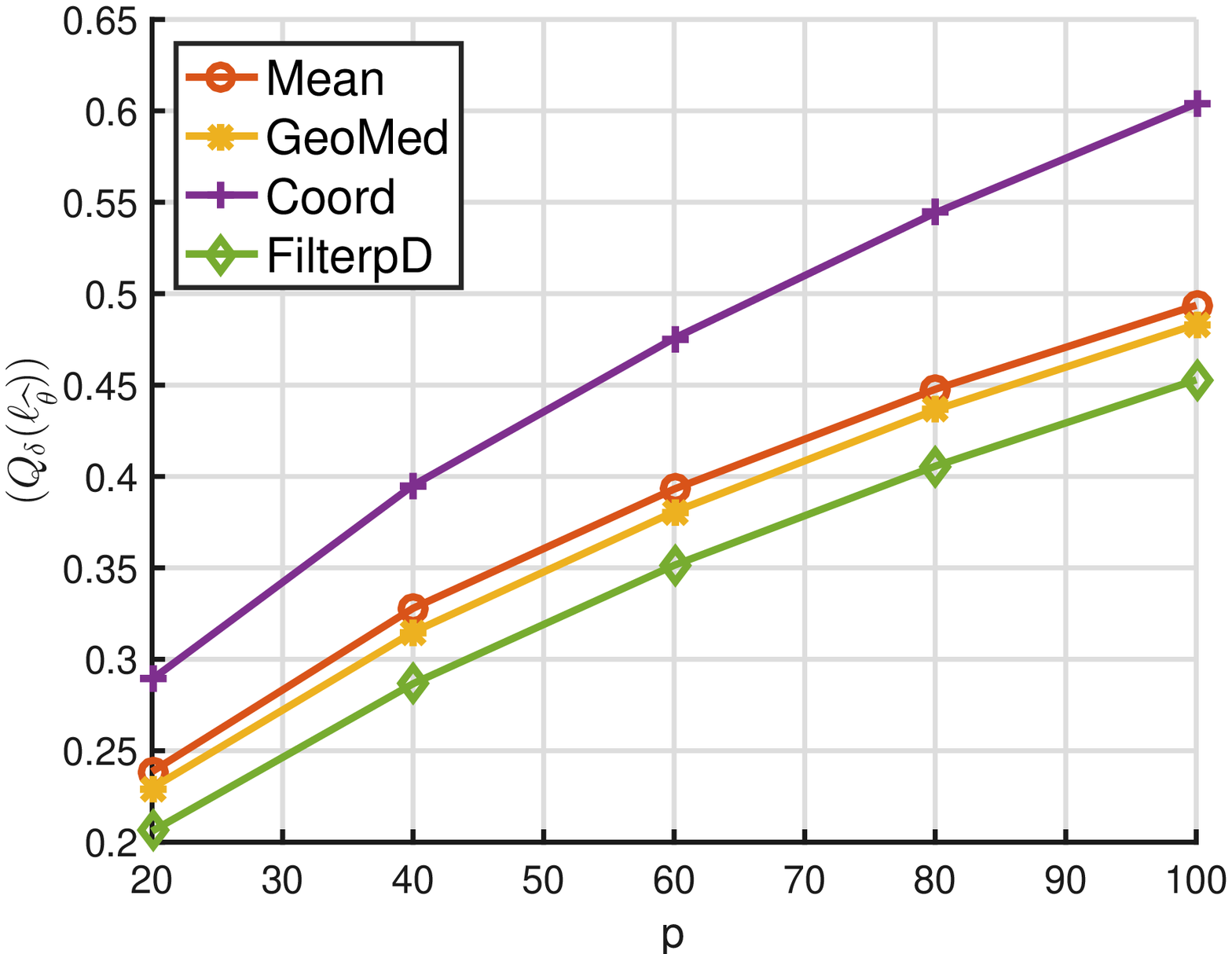}}
\caption{Mean Estimation for Multivariate Pareto Distribution}
\label{fig:pareto}
\end{figure}

\section{Conclusion}
In this work, we developed connections between Huber's $\epsilon$-contamination model and the heavy-tailed noise model. We studied conditions under which this connection leads to near-optimal procedures. Moreover, by building upon this connection, we leveraged recently proposed computationally-efficient algorithms for mean estimation in Huber's model to design new sample-pruning based efficient estimators for the heavy-tailed model. Our estimator is practical and easy to implement and we show that it performs better than other practical estimators such as Geometric Median of Means~\citep{minsker2015geometric}.

We also designed statistically-optimal estimators for general $2k$-moment bounded distributions, which simultaneously robust in both models. In particular, our results show that for $2k$-moment bounded distributions, one can achieve the optimal bias $O(\norm{\Sigma}{2}^{\half} \epsilon^{1 - 1/(2k)})$ at the optimal statistical sample complexity of $O(p)$. 

There are several avenues for future work. In particular, our computationally efficient estimators are such that when they achieve an optimal asymptotic bias in Huber's $\epsilon$-contamination they don't get the optimal deviation of $\opt_{n,\Sigma,\delta}$ in the Heavy-Tailed Model, and vice-versa. We leave designing computationally-efficient estimators which are simultaneously optimal as an open problem. 
Finally, it would also be of interest to design efficient estimators for different $\ell_q$-norms, and under additional structural assumptions such as sparsity and symmetry.

\clearpage
\begin{small}\bibliography{local}
\bibliographystyle{unsrtnat}
\end{small}
\clearpage


\appendix


\clearpage

\clearpage
\input{proofs}
\clearpage
\input{additionalProofs.tex}
\end{document}

%% file: proofs.tex
\section{Proofs}

\subsection{Proof of Theorem~\ref{thm:oracelArgument}}
\begin{proof}
 Using Chebyshev's inequality, we have that,
\[ \Pr \paren{\norm{x - \mu}{2} \geq R} \leq \frac{\Exp[\norm{x - \mu}{2}^{2k}]}{R^{2k}} \]
Now, to see that $\Exp[\norm{x - \mu}{2}^{2k}] \leq C \paren{\sqrt{\trace{\Sigma}}}^{2k}$. The case for $k=1$ is clear. We now show it for $k=2$.
\bit 

\item Let $\Sigma = Q \Lambda Q^T$ and $\{ q_i \}_{i=1}^p$ be the eigenvectors of $\Sigma$ and let $\lambda_i = q_i^T \Sigma q_i$ be the associated eigenvalue. Then, 
\begin{align}
    (x - \mu)^T(x - \mu) = \sum_i (q_i^T(x - \mu))^2 = \sum_i \nu_i^2,
\end{align}
where $\nu_i = q_i^T(x - \mu)$. Now, $\norm{x-\mu}{2}^4 = \paren{\sum \limits_{i}\nu_i^2}^2 = \sum_i \nu_i^4 + 2 \sum_{i \neq j} \nu_i^2 \nu_j^2$. Now, since we assume bounded fourth moments, we get that, $\Exp[\nu_i^4] \leq C (q_i^T\Sigma q_i)^2 = C \lambda_i^2$, Using Cauchy-Schwartz inequality, we get that $\Exp[\nu_i^2 \nu_j^2] \leq \sqrt{\Exp[\nu_i^4]}\sqrt{\Exp[\nu_j^4]} = C \lambda_i \lambda_j$. Hence, we have that,
\[ \Exp[\norm{x-\mu}{2}^{4}] \leq C \paren{\sum_i \lambda_i^2 + 2 \sum_{i \neq j} \lambda_i \lambda_j} = C_{4} \trace{\Sigma}^2 \]
 \eit 
 
\[ \Pr \paren{ \norm{x - \mu}{2} \geq R} \leq \frac{\Exp[\norm{x - \mu}{2}^{4}]}{R^4} = C_4 \frac{\trace{\Sigma}^2}{R^4}   \]
Hence, for $k={1,2}$, we have that, 
\begin{align}
    \Pr \paren{ \norm{x - \mu}{2} \geq R} \leq \frac{\paren{\sqrt{\trace{\Sigma}}}^{2k}}{R^{2k}}
\end{align}
Hence, know that for $x_i \sim P$, $\Pr(\calO(x_i) = 1) \geq 1 - \alpha$, where $\alpha = \frac{\paren{\sqrt{\trace{\Sigma}}}^{2k}}{R^{2k}}$.
Now, let  $G^0 \defeq \left\{ x_i~\st~ \calO(x_i) = 1 \right\}$. Then, using Bernstein's inequality, we know that with probability at least $1 - \delta$.
    \begin{align}
    n_{G^0} = |G^0| & \geq n \paren{1 - \alpha - C_1\sqrt{\alpha \frac{\log(1/\delta)}{n}} - C_2 \frac{\log(1/\delta}{n}} \\
    \geq n/2,
    \end{align}
where we make the assumption that, 
\[ \alpha + C_1 \sqrt{\alpha}\sqrt{\frac{\log(1/\delta)}{n}} + C_2 \frac{\log(1/\delta)}{n}  < \half  \]

\begin{enumerate}
\item Controlling $\norm{\mu - \Exp[\eparam_{G^0}]}{2}$ . This is a deterministic statement and essentially quantifies the amount the mean can shift, when the random variable is conditioned on an event. We show this in Claim~\ref{lem:approx_mean_pD} which was shown in \citep{steinhardt2018robust,lai2016agnostic}. We also provide a proof of the statement for completeness in Section~\ref{sec:proof_claim_approxmean_pD}.

\begin{claim}\label{lem:approx_mean_pD}[General Mean shift,\citep{steinhardt2018robust,lai2016agnostic}]  Suppose that a distribution $P$ has mean $\mu$ and covariance $\Sigma$ and bounded $2k$ moments. Then, for any event $\calA$ which occurs with probability at least $1 - \epsilon \geq \half$,
\begin{align}
    \norm{\mu - E[x|\calA]}{2} \leq 2 \norm{\Sigma}{2}^\half \epsilon^{1 - \frac{1}{2k}}
\end{align}
\end{claim}

Now using this Claim~\ref{lem:approx_mean_pD} with $\calA$ being the event that $\calO(x) = 1$,  we get that
            \begin{align} \label{eqn:oracle_1} \norm{\mu - \Exp[\eparam_{G^0}]}{2} \leq 2\norm{\Sigma}{2}^\half {\alpha}^{1 - 1/(2k)} 
            \end{align}

            \item \textbf{Controlling $\norm{\eparam_{G^0} - \Exp[\eparam_{G^0}]}{2}$}. This term measures how quickly the samples within $G^0$ converge to their true mean. To show this we use vector version of Bernstein's inequality. Let $z_i \defeq x_i - \Exp[\eparam_{G^0}]$ be the centered random variables. Then, we have that
            \begin{align*}
                \norm{z_i}{2} & \leq \norm{\tparam - \Exp[\eparam_{G^0}]}{2} + \norm{x_i - \tparam}{2} \\
                & \leq 2 \norm{\Sigma}{2}^\half \alpha^{1 - 1/(2k)} + R \\
                &  \leq 2R
            \end{align*}
            Similarly,
            \begin{align}
                \Exp[\norm{z_i}{2}^2] & = \Exp[\norm{x - E[x|\calA]}{2}^2| x \in \calA] \\
                & = \frac{\Exp[\norm{x - E[x|\calA]}{2}^2| \indic{x \in \calA}]}{P(\calA)}\\
                & \leq  2 \Exp[\norm{x - E[x|\calA]}{2}^2] \\
                & \leq 2 \Exp[\norm{x - E[x]}{2}^2] + 2 \norm{\tparam - E[x|\calA]}{2}^2 \\
                & \leq 2 \trace{\Sigma} + 4 \norm{\Sigma}{2} \alpha^{2 - 1/(k)} \\
                & \leq 4 \trace{\Sigma}
            \end{align}
            
            Now, we first state the vector version of Bernsteins inequality.

\begin{lemma}(Vector Bernstein, Corollary 8.45~\citep{foucart2013mathematical})
Let $Y_1,\ldots,Y_M$ be independent copies of a random vector $Y \in \calC^p$ satisfying $\Exp Y = 0$. Assume $\norm{Y}{2} \leq K$ for some $K>0$. Let, 
\[ Z = \norm{\sum \limits_{l=1}^M Y_l}{2}, \Exp[Z^2] = M \Exp[\norm{Y}{2}^2], \sigma^2 = \sup \limits_{\norm{v}{2} \leq 1} \Exp [|\inprod{v}{Y}|^2] \]
Then for $t > 0$,
\begin{align}
    \Pr \paren{Z \geq \sqrt{\Exp Z^2} + t} \leq \exp\paren{-\frac{t^2/2}{M\sigma^2 + 2K \sqrt{\Exp Z^2}+ tK/3}}
\end{align}
\end{lemma}
We use the above lemma, with $Y_i = \frac{z_i}{n_{G^0}}$. Hence, we have that, $K = \frac{2R}{n_{G^0}}$. Hence, we have that $Z = \norm{\sum\limits_{k=1}^{n_{G^0}}Y_k}{2} = \norm{\eparam_{G^0} - \Exp[\eparam_{G^0}]}{2}$. Hence, we have the following,
\bit 
\item $\Exp [Z^2] \leq n \frac{4\trace{\Sigma}}{n^2} = 4 {\frac{\trace{\Sigma}}{n}}$.
\item $ \sigma^2 \leq 4 \frac{\norm{\Sigma}{2}}{n^2}$. To see this, for any $v \in \calS^{p-1}$,
\begin{align*}
    \Exp[(v^T Y)^2] = \frac{1}{n^2} \Exp[(v^T (x - \mu_A))^2 | x \in \calA]
\end{align*}
where $\mu_A$ is the conditional mean, and $\calA$ is the event that $x~\st~ \norm{x - \mu}{2} \leq R$. We know that $P(A) \geq 1/2$. Hence, we get that,
\begin{align*}
   \Exp[(v^T Y)^2] & = \frac{1}{n^2} \frac{\Exp[(v^T (x - \mu_A))^2 \indic{x \in \calA}]}{P(A)} \\
   & \leq \frac{2}{n^2} \Exp[(v^T (x - \mu_A))^2] \\
   & =  \frac{2}{n^2} \paren{\Exp[(v^T (x - \mu))^2] +  \norm{\mu - \mu_A}{2}^2} \\
 \implies \sigma^2  & \leq \frac{2}{n^2} \paren{\norm{\Sigma}{2} +  \norm{\Sigma}{2} \alpha} \\
   & \leq \frac{4\norm{\Sigma}{2}}{n^2} 
\end{align*}
\eit 
Hence, we get that, with probability at least $1 - \delta$,
\begin{align*}
    \norm{\eparam_{G^0} - \Exp[\eparam_{G^0}]}{2} \leq & C_1 \sqrt{\frac{\trace{\Sigma}}{n_{G^0}}} + C_2 \norm{\Sigma}{2}^{\half} \sqrt{\frac{\log(1/\delta)}{n_{G^0}}} + C_3 R^\half \paren{\sqrt{\frac{\trace{\Sigma}}{n_{G^0}}}}^\half \sqrt{\frac{\log(1/\delta)}{n_G^0}} \nonumber \\
    & + C_4 R \frac{\log(1/\delta)}{n_{G^0}}
\end{align*}
Now, we use that $\sqrt{ab} \leq a + b~\forall~a,b \geq 0$. Hence, we get that with probability at least $1-\delta$ 
\[ \norm{\eparam_{G^0} - \Exp[\eparam_{G^0}]}{2} \leq  \underbrace{C_5 \sqrt{\frac{\trace{\Sigma}}{n_{G^0}}} + C_2 \norm{\Sigma}{2}^{\half} \sqrt{\frac{\log(1/\delta)}{n_{G^0}}}}_{T1} +  C_3 R \frac{\log(1/\delta)}{n_{G^0}}\]

\end{enumerate}

Using the bound on $\norm{\Exp[\eparam_{G^0}] - \mu}{2}$ from \eqref{eqn:oracle_1}, we get that, 

\begin{align}
    \norm{\eparam_{G^0} - \mu}{2} & \leq \norm{\Exp[\eparam_{G^0}] - \mu}{2} + \norm{\eparam_{G^0} - \Exp[\eparam_{G^0}]}{2} \\
    & \leq T1 + C_3 R \frac{\log(1/\delta)}{n_{G^0}} + 2\norm{\Sigma}{2}^{\half} \paren{\paren{\frac{\sqrt{\trace{\Sigma}}}{R}}^{2k}}^{1 - 1/(2k)} \\
    & = T1 + C_3 R \frac{\log(1/\delta)}{n_{G^0}} + 2\norm{\Sigma}{2}^{\half} \paren{\frac{(\sqrt{\trace{\Sigma}})^{2k-1}}{ R^{2k-1} }}
\end{align}
Under our assumption that $(\frac{\sqrt{\trace{\Sigma}}}{R})^{2k} + \frac{\log(1/\delta)}{n} < c$, we know that $n_{G^0} \geq n/2$. Hence, we get get that $T1 \precsim \opt_{n,\Sigma,\delta}$.
\end{proof}

\subsection{Proof of Corollary~\ref{cor:oracle4}}
The proof follows from plugging $R = \frac{\sqrt{\trace{\Sigma}}}{\sqrt{r(\Sigma)}^{1/4} \paren{\frac{\log(1/\delta)}{n}}^{1/4}}$ into Theorem~\ref{thm:oracelArgument}. 

\subsection{Proof of Corollary~\ref{cor:oracle2}}
The proof follows from plugging $R = \frac{\sqrt{\trace{\Sigma}}}{{r(\Sigma)}^{1/4} \paren{\frac{\log(1/\delta)}{n}}^{1/2}}$ into Theorem~\ref{thm:oracelArgument}. 

\clearpage

\subsection{Proof of Theorem~\ref{thm:filt_algo_pD}}

\begin{proof}
Our proof is split into two keys Lemmas. Firstly, in Lemma~\ref{lem:filt_stopping}, we show that the with probability at least $1 - \delta$, the algorithm terminates in at most $T_{\delta}^* = \ceil{18 \log(1/\delta) + 3 (n - n_{G^0})}$. Finally, in Lemma~\ref{lem:filt_connectMean} we show that under our assumptions that $8\frac{n - n_{G^0}}{n} +36\frac{\log(1/\delta)}{n}<\frac{1}{4}$, when the algorithm stops in $m = T_{\delta}^*$ steps, the sample mean of points, $\eparam_{S^m}$ is close to the mean of $G^0$. In particular, we show that
\begin{align} \label{eqn:connectMean}
\norm{\eparam_{G^0} - \eparam_{S^m}}{2} \leq C_1 \paren{8\frac{n-n_{G^0}}{n}+ 36\frac{\log(1/\delta)}{n}}^{\half} \norm{\Sigma_{G^0}}{2}^{\half},
\end{align}
which recovers the statement of the Theorem.

\begin{lemma}\label{lem:filt_stopping}
When Algorithm~\ref{algo:filteringpD} is instantiated on $S^0$, and provided with a known upper bound $C \norm{\Sigma_{G^0}}{2}$ then with probability $1 - \delta$, it stops in at most $T_{\delta}^* = \ceil{18 \log(1/\delta) + 3(n - n_{G^0}}$ steps.
\end{lemma}
\begin{proof}

At each step of Algorithm~\ref{algo:filteringpD}, we remove one sample based on the probability distribution of the scores. Let $l = 1,2,\ldots,n$ be the steps of the algorithm. Note that the steps of the Algorithm are dependent, hence to obtain a high probability statement, we will have to use martingale style analysis. The martingale analysis in the proof mostly follows from~\citep{xu2013outlier,liu2018high}. \\

Let $\calF^l$ be the filtration generated by the sets of events until step $l$. At step $l$, let $S^l$ be the set of samples, $G^l$ be the subset of $G^0$ stil in $S^l$, \ie $\{ x_i \in S^l \cap G^0 \}$.  Let $B^l = S^l {\backslash} G^l$ be the remaining samples. Note that $|S^l| = n_l = n - l$, and $S^l,G^l,B^l \in \calF^l$. \\

Let $\tau_i$ be some score for each point. Define $\calE^l$ be an event variable at step $l$ which is True if
\[ \sum \limits_{i \in G^l} \tau_i \geq \frac{1}{(\gamma - 1)} \sum \limits_{j \in B^{l}} \tau_j, \equiv \sum \limits_{i \in G^l} \tau_i \geq \frac{1}{\gamma} \sum \limits_{j \in S^{l}} \tau_j  \]
for say $\gamma = 3$. Intuitively, this means the event is true when the sum of the scores of the good points is larger compared to the bad points. Now, when $\calE^l$ is false, we sample a point $j$ according $\tau_j$ and remove it. Some algebra shows, that when $\calE^l$ is false, then with constant probability of $2/3$, we throw a point from $B^l$. 
\[ \Pr(\text{sample removed at Step $l$} \in B^l | \calF^l) = \frac{\sum \limits_{i \in B^l} \tau_i}{\sum_{j \in S^l} \tau_j} \geq \frac{\gamma - 1 }{\gamma} = 2/3 \]

Essentially, our argument shows that whenever $\calE^l$ is false, then we are more likely to throw a point from the bad set. This means, that in the next iteration the fraction of bad points will reduce. To argue more formally, let $T \defeq \min \{l : \calE^l~\text{is true} \}$ be the first time that $\calE^l$ is True. Then, our goal is to show that $T$ is small.

To show this, based on $T$, define $Y^l$, as
\[ Y^l = \begin{cases}
|B^{T-1}| + \frac{\gamma - 1}{\gamma}(T-1),~\text{if}~l \geq T \\
|B^{l}| + \frac{\gamma - 1}{\gamma} l,~\text{if}~l < T
\end{cases}\]
Now, we show that $\{Y^l,\calF^l\}$ is a supermartingale, \ie  $\Exp[Y^l | \calF^{l-1}] \leq Y^{l-1}$. To see this, we split it into three cases:
\bit 
\item \textbf{Case 1.} $l < T$. This means that $\calE^l$ is false. 
\begin{align}
    Y^l - Y^{l-1} = |B^l| - |B^{l-1}| + \frac{\gamma - 1}{\gamma},
\end{align}
Now, $|B^l| = |B^{l-1}|$ if no bad point is thrown, and $|B^{l}| = |B^{l-1}| - 1$ if the point thrown is bad. Since, $\calE^{l-1}$ is false, hence, we have that,
\[ \Exp[Y^l - Y^{l-1} | \calF^{l-1}]  = -1(\Pr(\text{sample removed at Step $l-1$} \in B^{l-1})) + \frac{\gamma - 1 }{\gamma} \substack{(i) \\ \leq} 0  \]
where $(i)$ is true because $\calE^{l-1}$ is false.
\item \textbf{Case 2.} $l=T$, This follows by construction, because at $l=T$, $Y^l = Y^{l-1}$.
\item \textbf{Case 3.} $l > T$, This also follows by construction.
\eit 
So, we have that ${Y^l,\calF^l}$ is a supermartingale. Now, we need to bound the steps $T_{\delta}$ such that the probability that the algorithm doesn't stop in $T_{\delta}$ steps is less than $\delta$, \ie
\[ \Pr \paren{\bigcap \limits_{l=1}^{T_{\delta}} \paren{\calE^l}^c} \leq \delta \]
Note, that,
\begin{align}
    \Pr\paren{\bigcap \limits_{l=1}^{T_{\delta}} \paren{\calE^l}^c} = \Pr\paren{T \geq T_{\delta}}  \substack{(ii) \\ \leq} \Pr\paren{Y^{T_{\delta}} \geq \frac{\gamma - 1}{\gamma} T_{\delta}}
\end{align}
where $(ii)$ follows because, if $T > T_{\delta} \implies Y^{T_{\delta}} = |B^{T_{\delta}}| + \frac{\gamma - 1}{\gamma}T_\delta \geq  \frac{\gamma - 1}{\gamma}T_\delta$.
Now, 
\begin{align}
\Pr \paren{Y^{T_\delta} \geq \frac{\gamma - 1}{\gamma} T_\delta} = \Pr \paren{Y^{T_\delta} - Y^0 \geq \frac{\gamma - 1}{\gamma} T_\delta - Y_0} \nonumber
\end{align}
Now, defining $D^l =  Y^l - Y^{l-1}$, and let $Z^l = D^l - \Exp[D^l | D^1,D^2,\ldots, D^{l-1}]$. Then,
\[ Y^{T_\delta} - Y^0 = \sum \limits_{l=1}^{T_\delta} D^l = \sum \limits_{l=1}^{T_\delta} Z^l + \sum \limits_{l=1}^{T_\delta} \Exp[D^l | D^1,D^2,\ldots,D^{l-1}]  \]
Since, we know that $\{Y^l,\calF^l\}$ is a supermartingale, hence the difference process is such that \[ \Exp[D^l | D^1,D^2,\ldots,D^{l-1}] \leq 0 \]
This implies that
\[ Y^{T_\delta} - Y^0 \leq \sum \limits_{l=1}^{T_\delta} Z^l \implies \Pr \paren{Y^{T_\delta} - Y^0 \geq \frac{\gamma - 1}{\gamma} T_\delta - Y_0} \leq \Pr \paren{ \sum \limits_{l=1}^{T_\delta} Z^l \geq \frac{\gamma - 1}{\gamma} T_\delta - Y_0} \]
Since, $|D^l| \leq 1$, and $Z^l \leq 2$ are bounded, hence we can use Azuma-Hoeffding to bound the above probability. In particular, 
\[ \Pr \paren{ \sum \limits_{l=1}^{T_\delta} Z^l \geq \frac{\gamma - 1}{\gamma} T_\delta - Y_0} \leq \exp\paren{ - \frac{\paren{\frac{\gamma - 1}{\gamma} T_\delta - Y_0}^2}{8T_\delta}}  \]
Now, we want a $T_\delta$ such that, $\exp\paren{ - \frac{\paren{\frac{\gamma - 1}{\gamma} T_\delta - Y_0}^2}{8T_\delta}} \leq \delta$. Solving the quadratic, we need a $T_\delta$ such that,
\[ \paren{\frac{\gamma - 1}{\gamma}}^2 T_\delta^2 - \paren{8 \log(1/\delta) + 2 Y^0 \frac{\gamma-1}{\gamma}} T_\delta + Y_0^2 \geq 0  \]
Some algebra shows that $T_{\delta}^* = \ceil{8 \log(1/\delta) \frac{\gamma^2}{(\gamma - 1)^2} + 2 Y^0 \frac{\gamma}{\gamma - 1}}$ satisifies the above equation. Hence, we know that with probability at least $1 - \delta$, there exists at least one good event in 1 to $T_{\delta}^*$ iterations. Note than $Y^0 = n_{B^0} = n - n_{G^0}$. 

While we have established that there is at least one good event in 1 to $T_{\delta}^*$ iterations, we need to show that whenever $\calE^{l}$ is True then Algorithm~\ref{algo:filteringpD} stops, \ie our checking condition is violated. To show this, we first prove that for $m \leq T_{\delta^*}$, when $\calE^m$ is true then  $\norm{\Sigma_{S^m}}{2} \leq 16 \norm{\Sigma_{G^m}}{2}$(See Claim~\ref{claim:cov1}). Coupling this with Claim~\ref{claim:cov2}, which shows that $\norm{\Sigma_{G^m}}{2} \leq 2 \norm{\Sigma_{G^0}}{2}$, we get that $\norm{\Sigma_{S^m}}{2} \leq 32 \norm{\Sigma_{G^0}}{2}$. Hence, Algorithm~\ref{algo:filteringpD} stops whenever $\calE^m$ is True.

\end{proof}

Next, we state and prove Lemma~\ref{lem:filt_connectMean}. Recall that $\calE^l$ is defined to be an event variable at step $l$ which is True if
\[ \sum \limits_{i \in G^l} \tau_i \geq \frac{1}{(\gamma - 1)} \sum \limits_{j \in B^{l}} \tau_j, \equiv \sum \limits_{i \in G^l} \tau_i \geq \frac{1}{\gamma} \sum \limits_{j \in S^{l}} \tau_j,  \]
where $S^l$ is set of samples at step $l$, and $G^l = \{ x_i \in S^l \cap G^0 \}$ is the subset of samples from $G^0$ which are still in $S^l$. Also, recall that for Algorithm~\ref{algo:filteringpD}, the sampling weights $\tau_i$ at any step $\ell$ are defined as $\tau_i = \paren{v^T(x_i - \eparam_{S^l})}^2$, where $v$ is the top unit-norm eigenvector of $\widehat{\Sigma}_{S^l}$ and $\eparam_{S^l}$ is the sample mean of $S^l$. Then, in Lemma~\ref{lem:filt_stopping} we showed that with probability $1 - \delta$, $\calE^m$ is True for some $m \leq T_{\delta^*} = \ceil{18 \log(1/\delta + 3(n - n_{G^0})}$.

\begin{lemma}\label{lem:filt_connectMean}
Let $\phi = \frac{n - n_{G^0}}{n}$. Then, under the assumption that $8\phi+36\frac{\log(1/\delta)}{n}<\frac{1}{4}$, we have that when $\calE^m$ is True, 
\[ \norm{\eparam_{G^0} - \eparam_{S^m}}{2} \leq 10\sqrt{2} \paren{8\phi+ 36\frac{\log(1/\delta)}{n}}^{\half} \norm{\Sigma_{G^0}}{2}^{\half}, \]
\end{lemma}
\begin{proof}
Using Lemma~\ref{lem:filt_meanControl}, we get that,
\[ \norm{\eparam_{G^0} - \eparam_{S^m}}{2} \leq \frac{\sqrt{TV(P_1,P_2)}}{1 - \sqrt{TV(P_1,P_2)}} \paren{\norm{\Sigma_{G^0}}{2}^{\half} + \norm{\Sigma_{S^m}}{2}^{\half}}, \]
where $P_1$ is the equal weight discrete distribution with support on $S^m$, and $P_2$ is the equal weight discrete distribution with support on $G^0$. In Claim~\ref{claim:TV} we show that \[ TV(P_1,P_2) \leq 8\phi + 36 \frac{\log(1/\delta)}{n} \]

When, $\calE^m$ is True, we know by contrapositive of Lemma~\ref{lem:score_filt} that $\norm{\Sigma_{S^m}}{2} \leq \frac{1 + \psi_m}{\frac{n_{S^m}}{n_{G^m}\gamma} - \psi_m} \norm{\Sigma_{G^m}}{2}$, where $\psi_m = \paren{\frac{\sqrt{TV(P_1,P_3)}}{1  - \sqrt{TV(P_1,P_3)}}}^2$. Coupling this with Claim~\ref{claim:cov2}, which shows that $\norm{\Sigma_{G^m}}{2} \leq 2 \norm{\Sigma_{G^0}}{2}$, we get that $\norm{\Sigma_{S^m}}{2} \leq 32 \norm{\Sigma_{G^0}}{2}$.
$$ \norm{\Sigma_{S^m}}{2} \leq C \norm{\Sigma_{G^0}}{2} $$ Hence, under our assumption that $8\phi+36\frac{\log(1/\delta)}{n}<\frac{1}{4}$, we get that,
\[ \norm{\eparam_{G^0} - \eparam_{S^m}}{2} \leq C \paren{8\phi+ 36\frac{\log(1/\delta)}{n}}^{\half} \norm{\Sigma_{G^0}}{2}^{\half}  \]
\end{proof}

\subsubsection{Auxillary Results for Proof of Theorem~\ref{thm:filt_algo_pD}}

\begin{lemma}\label{lem:score_filt}
Let $S$ be a collection of $n$ points. And let $G$ be a subset of $S$ containing $n_G$ points. Define $\tau_i = \paren{v^T(x_i - \eparam_S)}^2$, where $v$ is the top unit-norm eigenvector of $\widehat{\Sigma}_S$ and $\eparam_S$ is the sample mean of $S$. Let $\lambda = \norm{\Sigma_{S}}{2}$. Then, we have the following
\bit 
\item If $\lambda > \frac{1 + \psi}{\frac{n}{n_G\gamma} - \psi} \norm{\Sigma_G}{2}$,
\[ \sum \limits_{i: x_i \in G} \tau_i \substack{<} \frac{1}{\gamma} \sum \limits_{j=1}^n \tau_j,   \]
where $\psi = \paren{\frac{1}{\sqrt{\frac{n}{n-n_G}} - 1}}^2 < \frac{n}{n_G \gamma}$.
\eit 
\end{lemma}
\begin{proof}
Let $\eparam_G$ be the sample mean of points in $G$.
\begin{align}
    \frac{1}{n_G} \sum \limits_{i: x_i \in G} \tau_i &= \frac{1}{n_G} \sum \limits_{i: x_i \in G} v^T (x_i - \eparam_S)(x_i - \eparam_S)^T v \\
    &= v^T \paren{\frac{1}{n_G} \sum \limits_{i: x_i \in G} (x_i - \eparam_G)(x_i - \eparam_G)^T} v + \paren{v^T(\eparam_G - \eparam_S)}^2 \\
    & \leq v^T \Sigma_G v + \norm{\eparam_G - \eparam_S}{2}^2 \\
    & \leq v^T \Sigma_G v + \underbrace{\paren{\frac{1}{\sqrt{\frac{n}{n - n_G}}-1}}^2}_{\psi} \paren{\norm{\Sigma_S}{2} + \norm{\Sigma_G}{2}} \\
    & \leq \norm{\Sigma_G}{2} \paren{1 + \psi} + \psi \norm{\Sigma_S}{2}
\end{align}
Now, if $\norm{\Sigma_S}{2} \geq \frac{1 + \psi}{\frac{n}{n_G\gamma} - \psi} \norm{\Sigma_G}{2}$, then we have that
\begin{align}
    \frac{1}{n_G} \sum \limits_{i: x_i \in G} \tau_i & \leq \frac{n}{n_G \gamma} \norm{\Sigma_S}{2} \\
    & =  \frac{n}{n_G \gamma} \sum_{j=1}^n (v^T(x_j - \eparam_S))^2 \\ 
\implies \sum \limits_{i: x_i \in G} \tau_i & \leq \frac{1}{\gamma} \sum \limits_{j=1}^n \tau_j 
\end{align}
\end{proof}

\begin{claim}\label{claim:TV}
Suppose $P_1$ is the equal weight discrete distribution with support on $S^m$, and $P_2$ is the equal weight discrete distribution with support on $G^0$. Then, when $\phi = \frac{n_{B^0}}{n}$ is such that $3\phi + \frac{18\log(1/\delta)}{n} < \half$, 
\[ TV(P_1,P_2) \leq 8\phi + 36 \frac{\log(1/\delta)}{n} \]
\end{claim}
\begin{proof}
To bound the TV distance between $P_1$ and $P_2$, we use triangle inequality. Let $P_3$ be the equal weight discrete distribution with support on $G^m$. Let $\tau \in [T_\delta]$ be the number of "good" points thrown out in $T_\delta$ steps. For $\gamma = 3$, we have that,
\[ T_{\delta} = 18 \log(1/\delta) + 3n_{B^0} \]
\begin{align}
    TV(P_1,P_2)  & \leq TV(P_1,P_3) + TV(P_3,P_2) \\
    & \leq \frac{n_{S^m} - n_{G^m}}{n_{S^m}} + \frac{n_{G^0} - n_{G^m}}{n_{G^0}} \\
    & = \frac{n - T_{\delta} - (n-n_{B^0} - \tau)}{n-T_{\delta}} + \frac{\tau}{n - n_{B^0}} \\
    & = \frac{n_{B^0} + \tau - T_{\delta}}{n - T_{\delta}} + \frac{\tau}{n - n_{B^0}} \\
    & \leq \frac{n_{B^0}}{n - T_{\delta}} + \frac{T_{\delta}}{n - n_{B^0}} \\
    & = \frac{\phi}{1 - \frac{18 \log(1/\delta)}{n} - 3 \phi} + \frac{\frac{18 \log(1/\delta)}{n} + 3 \phi}{1 - \phi}
\end{align}
where $\phi = \frac{n_{B^0}}{n}$. Now under the assumption that $3\phi + \frac{18\log(1/\delta)}{n} < \half$, the first term is less than $2\phi$. 
\end{proof}

\begin{lemma}~\citep{kothari2018robust}\label{lem:filt_meanControl}
Given a collection of points $S$ of size $n$. Let $P_1$ and $P_2$ be discrete empirical distributions on $n$. Then, we have that, 
\begin{align}
    \norm{\Exp_{x_i \sim P_1}[x_i] - \Exp_{x_i \sim P_2[x_i]}}{2} \leq \frac{\sqrt{TV(P_1,P_2)}}{1- \sqrt{TV(P_1,P_2)}} \paren{\norm{\widehat{\Sigma}_{P_1}}{2}^{\half} + \norm{\widehat{\Sigma}_{P_2}}{2}^{\half}}
\end{align}
where $\widehat{\Sigma}_{P_1}$ is the covariance matrix when $x_i \sim P_1$, and $\widehat{\Sigma}_{P_2}$ is the empirical covariance matrix of when $x_i \sim P_2$
\end{lemma}
\begin{proof}

Consider a joint distribution(also called coupling) $\omega^*(z,z')$ over $S \times S$ such that it's individual marginal distributions are equal to $P_1$ and $P_2$; \ie $\omega(z) = P_1$ and $\omega(z') = P_2$ and $\omega(z \neq z') = TV(P_1,P_2)$. Then, we have that
\begin{align}
& \norm{\Exp_{x_i \sim P_1}[x_i] - \Exp_{x_i \sim P_2[x_i]} }{2} = \sup \limits_{v \in \calS^{{p-1}}} | \inprod{v}{\Exp_{w^*}\sqprn{z - z'}}| \\
&\leq  \sup_{v \in \calS^{{p-1}} } \Exp_{w^*}[ |\inprod{v}{z - z'}|] \\
& \leq \sup_{v \in \calS^{{p-1}} } \Exp_{w^*}[1(z \neq z')\inprod{v}{z - z'} | ] \\
& \leq (\Exp_{w^*}[(1(z \neq z'))^{1/(1 - \half)}])^{1 - \half} \sup_{v \in \calS^{{p-1}} } \Exp_{w^*}[(\inprod{v}{z - z'})^2]^{\half} \\
& \leq TV(P_1,P_2)^\half \sup_{v \in \calS^{{p-1}} } \paren{\Exp_{w^*}[(\inprod{v}{z - \Exp_{x_i \sim P_1}[x_i] + \Exp_{x_i \sim P_1}[x_i] - \Exp_{x_i \sim P_2}[x_i] + \Exp_{x_i \sim P_2}[x_i] - z'})^2]^{\half}} \\
& \leq TV(P_1,P_2)^\half \paren{\sup_{v \in \calS^{{p-1}} } \Exp_{w^*}[(\inprod{v}{z - \Exp_{x_i \sim P_1}[x_i]})^2]^\half + \norm{\Exp_{x_i \sim P_1}[x_i] - \Exp_{x_i \sim P_2}[x_i]}{2}} \nonumber \\
~~~~~~~& + TV(P_1,P_2)^\half \sup_{v \in \calS^{{p-1}} } \Exp_{w^*}[(\inprod{v}{z - \Exp_{x_i \sim P_2}[x_i]})^2]^\half \\
&\leq \frac{\sqrt{TV(P_1,P_2)}}{1 - \sqrt{TV(P_1,P_2)}} \left( \norm{\Sigma_{P_1}}{2}^{\half} + \norm{\Sigma_{P_2}}{2}^{\half} \right)
\end{align} 
\end{proof}

\begin{claim}\label{claim:cov2}
Under the assumption that $4\phi + 18\frac{\log(1/\delta)}{n} < \half$,  we have that, 
\[ \norm{\Sigma_{G^m}}{2} \leq 2 \norm{\Sigma_{G^0}}{2} \]
\end{claim}
\begin{proof}
We first show that $\norm{\Sigma_{G^m}}{2} \leq \frac{{n_{G^0}}}{n_{G^m}} \norm{\Sigma_{G^0}}{2}$.
\begin{align}
    \Sigma_{G^0} & = \frac{1}{n_{G^0}} \sum \limits_{i \in G^0} (x_i - \eparam_{G^0})(x_i - \eparam_{G^0})^T \\
    & = \frac{1}{n_{G^0}} \sum \limits_{i \in G^0} (x_i - \eparam_{G^0})(x_i - \eparam_{G^0})^T \paren{\indic{x_i \in G^m} + \indic{x_i \not \in G^m}} \\
    & = \frac{1}{n_{G^0}} \sum \limits_{i \in G^0} (x_i - \eparam_{G^0})(x_i - \eparam_{G^0})^T \paren{\indic{x_i \in G^m}} + \underbrace{\frac{1}{n_{G^0}} \sum \limits_{i \in G^0} (x_i - \eparam_{G^0})(x_i - \eparam_{G^0})^T \paren{ \indic{x_i \not \in G^m}}}_{T1} \\
    & = \frac{n_{G^m}}{n_{G^0}} \paren{\Sigma_{G^m} + (\eparam_{G^m} - \eparam_{G^0})(\eparam_{G^m} - \eparam_{G^0})^T} + T1
\end{align}
Now for $v$ being the top eigenvector of $\Sigma_{G^m}$, we get that,
\[ \frac{n_{G^m}}{n_{G^0}} v^T \Sigma_{G^m} v + \frac{n_{G^m}}{n_{G^0}} \underbrace{(v^T(\eparam_{G^m} - \eparam_{G^0}))^2}_{\geq 0} + \underbrace{v^T T1 v}_{\geq 0} = v^T \Sigma_{G^0} v  \]
Hence, we get that,
\[ \norm{\Sigma_{G^m}}{2} \leq \frac{n_{G^0}}{n_{G^m}} \norm{\Sigma_{G^0}}{2},  \]
Now,
\[ \frac{n_{G^0}}{n_{G^m}} = \frac{n- n_{B^0}}{n - n_{B^0} - \tau} \leq \frac{n- n_{B^0}}{n - n_{B^0} - T_{\delta}} = \frac{n- n_{B^0}}{n - 18 \log(1/\delta) - 4n_{B^0}} = \frac{1 - \phi}{1 - 18 \frac{\log(1/\delta)}{n} - 4 \phi}, \]
where $\phi = \frac{n_{B^0}}{n}$. Under our assumption, we get that, $\frac{n_{G^0}}{n_{G^m}} < 2$.
\end{proof}

 \begin{claim}\label{claim:cov1}
Under the assumption that $\phi = \frac{n_{B^0}}{n}$ is such that $3\phi + \frac{18\log(1/\delta)}{n} < \half$, and $2\phi < 0.12$, then when $\calE^m$ is True, we have that, 
\[ \norm{\Sigma_{S^m}}{2} \leq 16 \norm{\Sigma_{G^m}}{2} \]
\end{claim}
\begin{proof}
Suppose $P_1$ is the equal weight discrete distribution with support on $S^m$ and let $P_3$ be the equal weight discrete distribution with support on $G^m$. When $\calE^m$ is True, we know by contrapositive of Lemma~\ref{lem:score_filt} that $\norm{\Sigma_{S^m}}{2} \leq \frac{1 + \psi_m}{\frac{n_{S^m}}{n_{G^m}\gamma} - \psi_m} \norm{\Sigma_{G^m}}{2}$, where $\psi_m = \paren{\frac{\sqrt{TV(P_1,P_3)}}{1  - \sqrt{TV(P_1,P_3)}}}^2$. \\ 
Note that for $TV(P_1,P_3) = \frac{n_{S_m} - n_{G^m}}{n_{S^m}}$. Hence, $\frac{n_{S^m}}{n_{G^m}\gamma} = \frac{1}{\gamma(1 - TV(P_1,P_3))}$
For $\gamma=3$, the term $\frac{1 + \psi_m}{\frac{n_{S^m}}{n_{G^m}\gamma} - \psi_m}$ can be rewritten solely as a function of the $TV(P_1,P_3)$. In particular, it can be written as \[ f(x)  =  \frac{\left(1+\left(\frac{x^{0.5}}{1\ -\ x^{0.5}}\right)^2\right)\left(3\left(1-x^{0.5}\right)^2\left(1+x^{\left(0.5\right)}\right)\right)}{1-x^{\left(0.5\right)}-3x-3x^{\left(1.5\right)}} \]
Now $TV(P_1,P_3) = \frac{n_{S^m} - n_{G^m}}{n_{S^m}} = \frac{(n - T_\delta) - (n - n_{B^0} - \tau)}{n-T_{\delta}} = \frac{n_{B^0} + \tau - T_{\delta}}{n - T_{\delta}} \leq \frac{n_{B^0}}{n - T_{\delta}} = \frac{\phi}{1 - \frac{18 \log(1/\delta)}{n} - 3 \phi}$.
Hence, under our assumptions, $TV(P_1,P_3) < 0.12$. Some algebra shows that under $f(x)$ is monotonically increasing for $x < 0.12$, and in particular, $f(0.12) < 16$. Hence, we get that $\norm{\Sigma_{S^m}}{2} \leq 16 \norm{\Sigma_{G^m}}{2}$.
\end{proof}

\end{proof}

\clearpage

\subsection{Proof of Theorem~\ref{thm:oracle_cov}}

\begin{proof}
Let $G^0 = \{ x_ i | \calO(x_i) = 1 \}$ be the empirical collection of points chosen by the oracle. Let $n_{G^0} = |G^0|$. Then, we study and bound the operator norm of $\Sigma_{G^0}$. Recall that all oracles have the form $\indic{\norm{x_i - \mu}{2} \leq R}$, \ie, $ \forall x_i~\st~\calO(x_i) = 1$, we have that $\norm{x_i - \mu}{2} \leq R$.

Note that from Proof of Theorem~\ref{thm:oracelArgument}, we know that $\Pr(x \in G^0) \geq 1 - \alpha$, where $\alpha = \paren{\frac{\sqrt{\trace{\Sigma}}}{R}}^{2k}$. Let $\Sigma_{G^0}$ be the empirical covariance matrix. Then, 
\[ \Sigma_{G^0} = \frac{1}{n_{G^0}} \sum \limits_{i=1}^{n_{G^0}} (x_i - \eparam_{G^0})(x_i - \eparam_{G^0})^T  ,\]
where $\eparam_{G^0}$ is the empirical mean of the points in $G^0$. 
Recentering it around the true mean $\tparam$ of $P$, we get that,
\[ \Sigma_{G^0} = \frac{1}{n_{G^0}} \sum \limits_{i=1}^{n_{G^0}} (x_i - \tparam)(x_i - \tparam)^T - (\eparam_{G^0} - \tparam)(\eparam_{G^0} - \tparam)^T \]
Hence, we have that $\norm{\Sigma_{G^0}}{2} \leq \norm{\underbrace{\frac{1}{n_{G^0}}\sum \limits_{i=1}^{n_{G^0}} (x_i - \tparam)(x_i - \tparam)^T}_{A}}{2}$. To control, $\norm{A}{2}$, we use triangle inequality,
\begin{align}
    \norm{A}{2} \leq \underbrace{\norm{A - \Exp[A]}{2}}_{T1} + \underbrace{\norm{\Exp[A]}{2}}_{T2}
\end{align}
    
    \begin{enumerate}
        \item \textbf{Controlling T2.} Note that $\Exp[A] = \Exp[(x - \tparam)(x- \tparam)^T | x \in G]$.
        
        \begin{align}
            \Exp[A] & =  \frac{\Exp[(x - \tparam)(x - \tparam)^T \indic{x \in G^0}]}{P(x \in G^0)}
        \end{align}
        Let $\Pr(x \in G^0) \geq 1 - \alpha$. Hence, for any $v \in \calS^{p-1}$,
        \[ v^T \Exp[A] v = \frac{\Exp[(v^T(x - \tparam))^2 \indic{x \in G^0}]}{P(x \in G^0)} \leq \frac{\norm{\Sigma}{2}}{1 - \alpha} \]
    Under the assumption that $\alpha < \half$, we get that, 
    
    $$\norm{\Exp[A]}{2} \leq 2 \norm{\Sigma}{2}$$
    
    \item \textbf{Controlling T1.} Note that T1 can be controlled using a concentration of measure argument, and in particular exploits concentration of covariance for bounded random vectors.
    \begin{lemma}\label{lem:cov_heavy_versh}[Theorem 5.44~\citep{vershynin2010introduction}]
Let $\{y_i\}_{i=1}^n$ samples such that $y_i \in \real^p$ and $\norm{y_i}{2} \leq \sqrt{m}$ and $\Exp[yy^T] = \Sigma$. Then, with probability at least $1 - \delta$,
\[ \norm{\frac{1}{n}\sum_{i=1}^n y_iy_i^T - \Sigma}{2} \leq \max\paren{\norm{\Sigma}{2}^\half \sqrt{\log(p/\delta)}\sqrt{\frac{m}{n}},{\log(p/\delta)}{\frac{m}{n}}}  \]
\end{lemma}

    \begin{align}
    T1 & = \norm{\frac{1}{n_{G^0}} \sum \limits_{i=1}^{n_{G^0}} (x_i - \tparam)(x_i - \tparam)^T - \Exp[A]}{2} 
    \end{align}
    
    We use Lemma~\ref{lem:cov_heavy_versh} with $y_i = x_i - \tparam$. Note that $\sqrt{m} = R$. This means that with probability $1 - \delta$, 
        \begin{align*}
            T1 \leq C_1 R \norm{\Sigma}{2}^{\half} \sqrt{\frac{\log(p/\delta)}{n_{G^0}}} + R^2 \frac{\log(p/\delta)}{n_{G^0}}
        \end{align*} 
\end{enumerate}
Hence, we get that under the assumption that $\alpha + \sqrt{\alpha}\sqrt{\frac{\log(1/\delta)}{n}} < \half$, we recover statement of the result.
\end{proof}

\clearpage

\subsection{Proof of Corollary~\ref{lem:heavy_mean4}}
We consider the $\ell_2$ oracle of $\calO$ radius $R = \frac{\sqrt{\trace{\Sigma}}}{\paren{\frac{\log(1/\delta)}{n}}^{1/4}}$. Using chebychevs inequality, we know that $\Pr(\calO(x) = 1) \geq 1 - \alpha$, where $\alpha  = \frac{\log(1/\delta)}{n}$. 

Suppose we are given $n$-samples from $P$. Let $G^0$ be the set of points such that $\calO(x_i) = 1$. Using bernstein's inequality we know that with probability $1 - \delta$,
\begin{align}\label{eqn:L1}
    |n_{G^0} | & \geq n(1 - C \frac{\log(1/\delta)}{n})
\end{align}
Hence, we have that, 
\begin{align}\label{eqn:L2}
    \frac{n - n_{G^0}}{n} \lesssim \frac{\log(1/\delta)}{n}
\end{align}  
Let $\widehat{\mu}_n$ and $\Sigma_{G^0}$ be the empirical mean and covariance of the points in $G^0$.  

Let $\eparam_{\delta}$ be the output of Algorithm~\ref{algo:filteringpD}. Then, we know that with probability at least $1 - \delta$,
\begin{align}\label{eqn:L3}
    \norm{\eparam_{\delta} - \widehat{\mu}_n}{2} \lesssim \norm{\Sigma_{G^0}}{2}^{\half} \paren{\frac{n - n_{G^0}}{n} + \frac{\log(1/\delta)}{n}}^{\half}
\end{align}

Using Theorem~\ref{thm:oracle_cov}, we bound $\norm{\Sigma_{G^0}}{2}^{\half}$.
    \begin{align*}
    \norm{\Sigma_{n,\calO}}{2} \leq C_1 \norm{\Sigma}{2} + C_2 R \norm{\Sigma}{2}^{\half} \sqrt{\frac{\log(p/\delta)}{n_{G^0}}} + R^2 \frac{\log(p/\delta)}{n_{G^0}}
\end{align*}
\begin{align}\label{eqn:L5}
    \norm{\Sigma_{n,\calO}}{2}^{\half} \leq C_1 \norm{\Sigma}{2}^{\half} + C_2 R^{\half} \norm{\Sigma}{2}^{1/4} \paren{\frac{\log(p/\delta)}{n_{G^0}}}^{1/4}  + R \sqrt{\frac{\log(p/\delta)}{n_{G^0}}}
\end{align}
Plugging $R = \frac{\sqrt{\trace{\Sigma}}}{\paren{\frac{\log(1/\delta)}{n}}^{1/4}}$, we get,
\begin{align}\label{eqn:L6}
    \norm{\Sigma_{n,\calO}}{2}^{\half} \leq C_1 \norm{\Sigma}{2}^{\half} + \underbrace{C_2 \trace{\Sigma}^{1/4} \norm{\Sigma}{2}^{1/4} \frac{\paren{\frac{\log(p/\delta)}{n_{G^0}}}^{1/4}}{\paren{\frac{\log(1/\delta)}{n}}^{1/8}}}_{T1}  + \underbrace{\sqrt{\trace{\Sigma}} \frac{\sqrt{\frac{\log(p/\delta)}{n_{G^0}}}}{\paren{\frac{\log(1/\delta)}{n}}^{1/4}}}_{T2}
\end{align} 
Plugging~\eqref{eqn:L2} and \eqref{eqn:L6} into \eqref{eqn:L3}, we get that, 
\begin{align}
    \norm{\eparam_{\delta} - \widehat{\mu}_n}{2} \lesssim \norm{\Sigma}{2}^{1/2} \sqrt{\frac{\log(1/\delta)}{n}} + T1 \sqrt{\frac{\log(1/\delta)}{n}} + T2 \sqrt{\frac{\log(1/\delta)}{n}}
\end{align}
When $T1$ and $T2$ are less than $C \sqrt{\norm{\Sigma}{2}}$, then we have that,
\begin{align}
    \norm{\eparam_{\delta} - \widehat{\mu}_n}{2} \lesssim \norm{\Sigma}{2}^{1/2} \sqrt{\frac{\log(1/\delta)}{n}}
\end{align}
Some algebra shows that when $\frac{r(\Sigma)^2 \log^2(p/\delta)}{n \log(1/\delta)} \leq C$, the both $T1$ and $T2$ are $O(\sqrt{\norm{\Sigma}{2}})$. Hence, we get that,
\begin{align}\label{eqn:L7}
    \norm{\eparam_{\delta} - \widehat{\mu}_n}{2} \lesssim \norm{\Sigma}{2}^{1/2} \sqrt{\frac{\log(1/\delta)}{n}}
\end{align}
Using Theorem~\ref{thm:oracelArgument}, and plugging $R = \frac{\sqrt{\trace{\Sigma}}}{\paren{\frac{\log(1/\delta)}{n}}^{1/4}}$, we get that with probability at least $1-\delta$,
\begin{align}\label{eqn:L8}
\norm{\mu(P) - \widehat{\mu}_n}{2} \lesssim \opt_{n,\Sigma,\delta} + \underbrace{\sqrt{\trace{\Sigma}} \paren{\frac{\log(1/\delta)}{n}}^{3/4}}_{T3}
\end{align}
Under our assumption that $r^2(\Sigma) \frac{\log(1/\delta)}{n} \leq C$, $T3 \lesssim  \norm{\Sigma}{2}^{1/2} \sqrt{\frac{\log(1/\delta)}{n}}$. Combining the above equation and ~\ref{eqn:L7}, we recover the corollary statement.

\clearpage
\subsection{Proof of Corollary~\ref{lem:heavy_mean2}}
We consider the $\ell_2$ oracle of $\calO$ radius $R = \frac{\sqrt{\trace{\Sigma}}}{\paren{\frac{\log(1/\delta)}{n}}^{1/2}}$. Using chebychevs inequality, we know that $\Pr(\calO(x) = 1) \geq 1 - \alpha$, where $\alpha  = \frac{\log(1/\delta)}{n}$. 

Suppose we are given $n$-samples from $P$. Let $G^0$ be the set of points such that $\calO(x_i) = 1$. Using bernstein's inequality we know that with probability $1 - \delta$,
\begin{align}\label{eqn:Q1}
    |n_{G^0} | & \geq n(1 - C \frac{\log(1/\delta)}{n})
\end{align}
Hence, we have that, 
\begin{align}\label{eqn:Q2}
    \frac{n - n_{G^0}}{n} \lesssim \frac{\log(1/\delta)}{n}
\end{align}  
Let $\widehat{\mu}_n$ and $\Sigma_{G^0}$ be the empirical mean and covariance of the points in $G^0$.  

Let $\eparam_{\delta}$ be the output of Algorithm~\ref{algo:filteringpD}. Then, we know that with probability at least $1 - \delta$,
\begin{align}\label{eqn:Q3}
    \norm{\eparam_{\delta} - \widehat{\mu}_n}{2} \lesssim \norm{\Sigma_{G^0}}{2}^{\half} \paren{\frac{n - n_{G^0}}{n} + \frac{\log(1/\delta)}{n}}^{\half}
\end{align}

Using Theorem~\ref{thm:oracle_cov}, we bound $\norm{\Sigma_{G^0}}{2}^{\half}$.
    \begin{align*}
    \norm{\Sigma_{n,\calO}}{2} \leq C_1 \norm{\Sigma}{2} + C_2 R \norm{\Sigma}{2}^{\half} \sqrt{\frac{\log(p/\delta)}{n_{G^0}}} + R^2 \frac{\log(p/\delta)}{n_{G^0}}
\end{align*}
\begin{align}\label{eqn:Q5}
    \norm{\Sigma_{n,\calO}}{2}^{\half} \leq C_1 \norm{\Sigma}{2}^{\half} + C_2 R^{\half} \norm{\Sigma}{2}^{1/4} \paren{\frac{\log(p/\delta)}{n_{G^0}}}^{1/4}  + R \sqrt{\frac{\log(p/\delta)}{n_{G^0}}}
\end{align}
Plugging $R = \frac{\sqrt{\trace{\Sigma}}}{\paren{\frac{\log(1/\delta)}{n}}^{1/2}}$, we get,
\begin{align}\label{eqn:Q6}
    \norm{\Sigma_{n,\calO}}{2}^{\half} \leq C_1 \norm{\Sigma}{2}^{\half} + C_2 \trace{\Sigma}^{1/4} \norm{\Sigma}{2}^{1/4} \paren{\frac{\log(p/\delta)}{\log(1/\delta)}}^{1/4}  + \frac{\sqrt{\trace{\Sigma}}}{\sqrt{\frac{\log(1/\delta)}{n}}} \sqrt{\frac{\log(p/\delta)}{n}}
\end{align} 
Plugging~\eqref{eqn:Q2} and \eqref{eqn:Q6} into \eqref{eqn:Q3}, we get that, 
\begin{align}\label{eqn:Q7}
    \norm{\eparam_{\delta} - \widehat{\mu}_n}{2} \lesssim \norm{\Sigma}{2}^{1/2} \sqrt{\frac{\log(1/\delta)}{n}} + \sqrt{\frac{\trace{\Sigma}\log(p/\delta)}{n}}
\end{align}
Using Theorem~\ref{thm:oracelArgument}, and plugging $R = \frac{\sqrt{\trace{\Sigma}}}{\paren{\frac{\log(1/\delta)}{n}}^{1/2}}$, we get that with probability at least $1-\delta$,
\begin{align}\label{eqn:Q8}
\norm{\mu(P) - \widehat{\mu}_n}{2} \lesssim \opt_{n,\Sigma,\delta} + \sqrt{\frac{\trace{\Sigma}\log(1/\delta)}{n}}
\end{align}
Combining the above equation and ~\ref{eqn:Q7}, we recover the corollary statement.

\clearpage

\subsection{Proof of Corollary~\ref{lem:filt_algo_pD_2}}
We follow along the lines of the proof of Corollary~\ref{lem:heavy_mean2}. 
 Suppose we are given $n$-samples from $P_{\epsilon}$. Let $G^0$ be the set of points such that $x_i \sim P~~\text{and}~~\calO(x_i) = 1$, where $\calO(\cdot)$ is an $\ell_2$ oracle of radius $R = \frac{\sqrt{\trace{\Sigma}}}{(\epsilon + \frac{\log(1/\delta)}{n})^{1/2}}$.

Consider the event $E_1 = x \sim P$, then $P_\epsilon(E_1) = 1 - \epsilon$. Consider the event $E_2 = \norm{x_i - \tparam}{2} \leq \sqrt{\frac{\trace{\Sigma}}{{\epsilon + \frac{\log(1/\delta)}{n}}}}$. We have that $P_{\epsilon}(E_2|E_1) = 1 - \Pr\paren{\norm{x - \tparam}{2} > \sqrt{\frac{\trace{\Sigma}}{{\epsilon + \frac{\log(1/\delta)}{n}}}}| x \sim P}$. Using Chebyshev's inequality, we have that,
$$ P^*(\norm{x - \mu}{2} > {\frac{\sqrt{\trace{\Sigma}}}{\sqrt{\epsilon +\frac{\log(1/\delta)}{n}}}}) \leq \frac{\Exp[\norm{x - \mu}{2}^2]}{\paren{{\frac{\sqrt{\trace{\Sigma}}}{\epsilon +\sqrt{\frac{\log(1/\delta)}{n}}}}}^2} = \epsilon +\frac{\log(1/\delta)}{n}$$.
\begin{align} 
P_{\epsilon}(E_1 \cap E_2) & \geq (1 - \epsilon)(1 - \epsilon - \frac{\log(1/\delta)}{n})  \\
& \geq 1 - C (\epsilon + \frac{\log(1/\delta)}{n})
\end{align}
Now, given $n$-samples from $P_\epsilon$, we use Bernsteins bound to get the empirical probability, \ie we get that with probability at least $1 - \delta$
\begin{align} 
P_{\epsilon}(E_1 \cap E_2) - P_{n,\epsilon}(E_1 \cap E_2) & \leq C_1 \sqrt{\epsilon + \frac{\log(1/\delta)}{n}}\sqrt{\frac{\log(1/\delta)}{n}} + C_2 \frac{\log(1/\delta)}{n} \\
& \lesssim \epsilon + \frac{\log(1/\delta)}{n}
\end{align}
\begin{align}
\implies n_{G^0} & \geq n\paren{1 - \epsilon - \frac{\log(1/\delta)}{n}} \geq n/2
\end{align}
The remaining proof follows along the lines of Corollary~\ref{lem:heavy_mean2}. Using Theorem~\ref{thm:oracle_cov}
\begin{align}\label{eqn:R6}
    \norm{\Sigma_{n,\calO}}{2}^{\half} \lesssim \norm{\Sigma}{2}^{\half} + \sqrt{\frac{\trace{\Sigma}\log(p/\delta)}{n\epsilon + \log(1/\delta)}}
\end{align}
Hence, we get that,
\begin{align}\label{eqn:R7}
    \norm{\eparam_{\delta} - \widehat{\mu}_n}{2} \lesssim \norm{\Sigma}{2}^{1/2} \sqrt{\epsilon} + \norm{\Sigma}{2}^{1/2} \sqrt{\frac{\log(1/\delta)}{n}} +  \sqrt{\frac{\trace{\Sigma}\log(p/\delta)}{n}}
\end{align}
Using Theorem~\ref{thm:oracelArgument}, and plugging $R = \frac{\sqrt{\trace{\Sigma}}}{\paren{\epsilon + \frac{\log(1/\delta)}{n}}^{1/2}}$, we get that with probability at least $1-\delta$,
\begin{align}\label{eqn:R8}
\norm{\mu(P) - \widehat{\mu}_n}{2} \lesssim \norm{\Sigma}{2}^{1/2} \sqrt{\epsilon} + \opt_{n,\Sigma,\delta} + \sqrt{\frac{\trace{\Sigma}\log(p/\delta)}{n}}
\end{align}
Combining the above equation and~\ref{eqn:R7}, we recover the corollary statement.

\clearpage
\subsection{Proof of Corollary~\ref{lem:filt_algo_pD_4}}
Suppose we are given $n$-samples from $P_{\epsilon}$. Let $G^0$ be the set of points such that $x_i \sim P~~\text{and}~~\calO(x_i) = 1$, where $\calO(\cdot)$ is an $\ell_2$ oracle of radius $R = \frac{\sqrt{\trace{\Sigma}}}{(\epsilon + \frac{\log(1/\delta)}{n})^{1/4}}$.

Consider the event $E_1 = x \sim P$, then $P_\epsilon(E_1) = 1 - \epsilon$. Consider the event $E_2 = \norm{x_i - \tparam}{2} \leq {\frac{\sqrt{\trace{\Sigma}}}{(\epsilon + \frac{\log(1/\delta)}{n})^{1/4}}}$. We have that $P_{\epsilon}(E_2|E_1) = 1 - \Pr\paren{\norm{x - \tparam}{2} > {\frac{\sqrt{\trace{\Sigma}}}{\paren{\epsilon + \frac{\log(1/\delta)}{n}}^{1/4}}}| x \sim P}$. Using Chebyshev's inequality, we have that,
$$ P^*(\norm{x - \mu}{2} > {\frac{\sqrt{\trace{\Sigma}}}{\paren{\epsilon + \frac{\log(1/\delta)}{n}}^{1/4}}} \leq \frac{\Exp[\norm{x - \mu}{2}^4]}{\paren{{\frac{\sqrt{\trace{\Sigma}}}{\epsilon + {\frac{\log(1/\delta)}{n}}}}}^4} = \epsilon +\frac{\log(1/\delta)}{n}$$.
\begin{align} 
P_{\epsilon}(E_1 \cap E_2) & \geq (1 - \epsilon)(1 - \epsilon - \frac{\log(1/\delta)}{n})  \\
& \geq 1 - C (\epsilon + \frac{\log(1/\delta)}{n})
\end{align}
Now, given $n$-samples from $P_\epsilon$, we use Bernsteins bound to get the empirical probability, \ie we get that with probability at least $1 - \delta$
\begin{align} 
P_{\epsilon}(E_1 \cap E_2) - P_{n,\epsilon}(E_1 \cap E_2) & \leq C_1 \sqrt{(\epsilon)}\sqrt{\frac{\log(1/\delta)}{n}} + C_2 \frac{\log(1/\delta)}{n} \\
& \lesssim \epsilon + \frac{\log(1/\delta)}{n}
\end{align}
\begin{align}
\implies n_{G^0} & \geq n\paren{1 - \epsilon - \frac{\log(1/\delta)}{n}} \geq n/2
\end{align}

Let $\eparam_{\delta}$ be the output of Algorithm~\ref{algo:filteringpD}. Then, we know that with probability at least $1 - \delta$,
\begin{align}\label{eqn:S3}
    \norm{\eparam_{\delta} - \widehat{\mu}_n}{2} \lesssim \norm{\Sigma_{G^0}}{2}^{\half} \paren{\frac{n - n_{G^0}}{n} + \frac{\log(1/\delta)}{n}}^{\half}
\end{align}

Using Theorem~\ref{thm:oracle_cov}, we bound $\norm{\Sigma_{G^0}}{2}^{\half}$.
    \begin{align*}
    \norm{\Sigma_{n,\calO}}{2} \leq C_1 \norm{\Sigma}{2} + C_2 R \norm{\Sigma}{2}^{\half} \sqrt{\frac{\log(p/\delta)}{n_{G^0}}} + R^2 \frac{\log(p/\delta)}{n_{G^0}}
\end{align*}
\begin{align}\label{eqn:S5}
    \norm{\Sigma_{n,\calO}}{2}^{\half} \lesssim  \norm{\Sigma}{2}^{\half} + R \sqrt{\frac{\log(p/\delta)}{n_{G^0}}}
\end{align}
Plugging $R = \frac{\sqrt{\trace{\Sigma}}}{\paren{\epsilon + \frac{\log(1/\delta)}{n}}^{1/4}}$, we get,
\begin{align}\label{eqn:S6}
    \norm{\Sigma_{n,\calO}}{2}^{\half} \lesssim   \norm{\Sigma}{2}^{\half} + \frac{\sqrt{\trace{\Sigma}}}{\paren{\epsilon + \frac{\log(1/\delta)}{n}}^{1/4}} \sqrt{\frac{\log(p/\delta)}{n_{G^0}}}
\end{align} 
Plugging~\eqref{eqn:S6} into \eqref{eqn:S3}, we get that, 
\begin{align}
    \norm{\eparam_{\delta} - \widehat{\mu}_n}{2} \lesssim \norm{\Sigma}{2}^{1/2} \sqrt{\epsilon} + {\sqrt{\trace{\Sigma}}} \sqrt{\frac{\log(p/\delta)}{n_{G^0}}}\paren{\epsilon + \frac{\log(1/\delta)}{n}}^{1/4}
\end{align}

Using Theorem~\ref{thm:oracelArgument} by plugging $R = \frac{\sqrt{\trace{\Sigma}}}{\paren{\epsilon + \frac{\log(1/\delta)}{n}}^{1/4}}$, and then using triangle inequality, we get that with probability at least $1-\delta$,
\begin{align}\label{eqn:S8}
\norm{\mu(P) - \widehat{\mu}_n}{2} \lesssim \norm{\Sigma}{2}^{1/2} \sqrt{\epsilon} + \opt_{n,\Sigma,\delta} + {\sqrt{\trace{\Sigma}}} \sqrt{\frac{\log(p/\delta)}{n_{G^0}}}\paren{\epsilon + \frac{\log(1/\delta)}{n}}^{1/4}
\end{align}

\clearpage

%% file: additionalProofs.tex
\section{Additional Proofs}\label{app:opt_proofs}

\subsection{Proof of Claim~\ref{lem:approx_mean_pD}}\label{sec:proof_claim_approxmean_pD}

\begin{proof}
For any event $\calA$, Let $\indic{\calA}$ denote the corresponding indicator variable. 
\begin{align}
    \norm{E_{x \sim P} [x | \calA] - \mu]}{2} = \frac{1}{P(\calA)}\norm{{ E_{x \sim P^*}((x - \mu ) \indic{\calA})}}{2} \leq 2 \norm{{ E_{x \sim P^*}((x - \mu ) \indic{\calA})}}{2},
\end{align}
\begin{align}
 \Exp_{x \sim P}[(x - \mu) \indic{x \in \calA^c} + (x - \mu) \indic{x \in \calA}] & = \Exp_{x \sim P}[(x - \mu)] = 0 \\
 \implies \norm{\Exp_{x \sim P}[(x - \mu) \indic{x \in \calA^c}]}{2} & = \norm{\Exp_{x \sim P}[(x - \mu) \indic{x \in \calA}]}{2}
 \end{align}
    \begin{align} \norm{\Exp_{x \sim P}[(x - \mu) \indic{x \in \calA^c}]}{2} & = \sup \limits_{u \in \calS^{p-1}} |\Exp_{(x \sim P}[u^T(x - \mu) \indic{x \in \calA^c}] | \\
    &  \substack{{(i)} \\ {\leq}} \sup \limits_{u \in \calS^{p-1}} \paren{E_{x \sim P}[\paren{u^T (x - \mu)}^{2k}]}^{1/(2k)} \paren{E_{x \sim P}[\indic{x \in \calA^c}^{1-1/(2k)}]}^{1 - \frac{1}{2k}} \\
    & \leq C_{2k}^{1/(2k)} \norm{\Sigma}{2}^{\half} \epsilon^{1 - 1/(2k)} 
    \end{align}
    where (i) follows from Holder's inequality.
\end{proof}

\clearpage

\clearpage

\subsection{Proof of Lemma~\ref{lem:convex_huber}}

\begin{proof}
Let $P = \calN(0,\calI_p)$ be the isotropic normal distribution. Let $R_P(\theta) = \Exp_{z \sim P}[\ell(\norm{z - \theta}{2})]$, where $\ell:\real \mapsto \real$ is a convex loss, and let $\theta(P) = \argmin_{\theta} R_P(\theta)$ be the minimizer of the population risk. We assume that $\psi(\cdot) = \ell'(\cdot) < C$ is bounded. Note that when the derivative is unbounded, it is easy to argue that the corresponding risk will be non-robust. We also assumed that this risk is fisher-consistent for the Gaussian-distribution, \ie $\theta(P) = 0$. For notational convenience, let $u(t) = \frac{\psi(t)}{t}$.
Then,
\[ \grad R_{P}(\theta) = - \Exp_{z \sim P } \sqprn{ \underbrace{\frac{\psi(\norm{z - \theta}{2})}{\norm{z - \theta}{2}}}_{u(\norm{z - \theta}{2})}(z - \theta)}. \]
As before, let $P_\epsilon = (1 - \epsilon) P + \epsilon Q$. Then, we are interested in studying $\eparam(P_\epsilon)$. To do this, by first order optimality, we know that $\theta(P_\epsilon)$ is a solution to the following equation:
\[ (1 - \epsilon) \grad R_P(\theta(P_\epsilon)) + \epsilon \grad R_Q(\theta(P_\epsilon)) = 0 \]
First we calculate the derivative of $\theta(P_\epsilon)$ \wrt $\epsilon$ using the fixed point above. Taking derivative of the above equation \wrt $\epsilon$
\begin{align}
    (1- \epsilon) \grad^2 R_{P}(\theta(P_\epsilon))\dot{\theta}(P_\epsilon)   - \grad R_P(\theta(P_\epsilon)) + \epsilon \grad^2 R_{Q}(\theta(P_\epsilon))\dot{\theta}(P_\epsilon) + \grad R_{Q}(\theta(P_\epsilon)) = 0
    \end{align}
    Under our assumption that $\psi$ is continuous, we get that at $\epsilon = 0$, 
   
    \begin{align}\label{eqn:thetaDotEquation}
    \dot{\theta}(P_\epsilon)_{| \epsilon = 0} =  \paren{ - \grad^2 R_{P}(\theta(P))}\inv \grad R_{Q}(\theta(P))
\end{align}
By fisher consistency of $\ell$ for $\calN(0,\calI_p)$, we have that $\theta(P) = 0$. Suppose that $Q$ is a point mass distribution with all mass on $\theta_Q$. Then, we have that,
\[ \grad R_{Q} (0) =  - u(\norm{\theta_Q}{2}) \theta_Q \]
Our next step is to lower bound the operator norm of $- \grad^2 R_{P}(\theta(P))$. To do this we show that for any unit vector $v \in \calS^{p-1}$, $v^T (- \grad^2 R_{P}(\theta(P)))v \leq \frac{C_2}{\sqrt{p}}$. 
\[ \grad^2 R_{P}(\theta) = - \Exp_{z \sim P} \sqprn{u(\norm{z - \theta}{2}) \calI_p + \frac{u'(\norm{z - \theta}{2}) }{\norm{z - \theta}{2}}{((z - \theta)(z - \theta)^T)}} \]
Now, by definition $u(t) = \psi(t)/t$, so $u'(s) = ( \psi'(s) - u(s))/s$. Plugging this above, 
\[ \grad^2 R_{P}(\theta) = - E_{z \sim P} \sqprn{ u(\norm{z - \theta}{2})\paren{ \calI_p - \frac{(z - \theta)(z - \theta)^T)}{\norm{z - \theta}{2}^2}} + \frac{\psi'(\norm{z - \theta}{2})}{\norm{z - \theta}{2}^2}(z - \theta)(z - \theta)^T))}   \]
 Hence, we get that
\[ v^T \grad^2 R_{P}(0) v = - \Exp_{z \sim N(0,I_p)} \sqprn{ u(\norm{z}{2}) \paren{ \norm{v}{2}^2 - (v^T (z/ \norm{z}{2}))^2}
 + \psi'(\norm{z}{2})(v^T (z/ \norm{z}{2}))^2
}  \]
Further for Isotropic Gaussian, $\norm{z}{2}$ and $z/\norm{z}{2}$ are independent random variables. Also, since, $z/\norm{z}{2}$ is uniformly distributed on unit sphere, we get that $\Exp_{z \sim N(0,I)} [(v^T z/\norm{z}{2})^2)] = \norm{v}{2}^2/p$.
\[ (v^T ( - \grad^2 R_{P}(0) ) v) = \underbrace{\Exp_{z \sim N(0,I_p)} \sqprn{ u(\norm{z}{2})}(1 - 1/p)}_{\textbf{T1}} + \underbrace{\Exp_{z \sim N(0,I_p)} \sqprn{ \psi'(\norm{z}{2})}/p}_{\textbf{T2}}  \]
\bit 
\item \textbf{Controlling T1} 
\begin{align}
    \Exp_{z \sim N(0,I_p)} [u(\norm{z}{2})] & = \Exp_{z \sim \calN(0,I_p)} \sqprn{\frac{\psi(\norm{z}{2})}{\norm{z}{2}}} \nonumber \\
    & \leq \sqrt{C \Exp{\frac{1}{\norm{z}{2}^2}}} \nonumber \\
    & \leq \frac{\sqrt{C_1}}{\sqrt{p-2}},
\end{align}
where we use that $\psi$ is bounded by constant $C$. The last inequality is combination of Jensen's Inequality and plugging the mean of reciprocal of inverse chi-squared random variable~\citep{bernardo2009bayesian}.
\item \textbf{Controlling T2.} Under our assumption that $\psi'(\cdot)$ exists and is bounded, we get that $T2 \leq \frac{C_1}{p}$ and can be ignored.
\eit 
Hence, for large $p$, we get that $(v^T ( - \grad^2 R_{P}(0) ) v) \leq \sqrt{C_1/p}$. Now, if we put $\theta_Q$ at $\infty$, and use that $\psi(\infty) = C_1$, we get that, 
\[ \norm{\dot{\theta}(P_\epsilon)}{2} = \psi(\norm{\theta_Q}{2}) \norm{\grad^2 R_{P}(0) \frac{\theta_Q}{\norm{\theta_Q}{2}}}{2} \geq C_2 \sqrt{p} \]

\end{proof}

\clearpage

\subsection{Proof of Lemma~\ref{lem:srm_mean}}
\begin{proof}
Let $P = N(0,\calI_p)$. Every subset of size $(1 - \epsilon)n$ can be thought of as samples from a mixture distribution defined in \eqref{eqn:huber_mixture}, where the mixture proportion $\eta$, ranges from $[0, \epsilon/(1 - \epsilon)]$. In the asymptotic setting of $n \mapsto \infty$, the empirical squared loss over each subset corresponds to the population risk with the sampling distribution as $P_\eta$. For a given contamination distribution $Q$, let $R_{P_\eta}(\theta) = \Exp_{x \sim P_{\eta}} \sqprn{\norm{x - \theta}{2}^2}$ and let $\theta(P_\eta) \defeq \argmin_\theta R_{P_\eta}(\theta)$, then subset risk minimization returns, 
\begin{align}
& \eSRM = \theta (P_{\eta^*}) \\
& \text{where} \ \eta^* = \argmin_{\eta \in [0,\frac{\epsilon}{1 - \epsilon}]} R_{P_\eta}(\theta(P_\eta)) \nonumber
\end{align}
We are interested in bounding the bias of SRM \ie
\[
\sup \limits_Q \norm{\eSRM - \tparam}{2}
\]
To do this, we know that for any contamination distribution $Q$, the solution of SRM necessarily satisfies the following conditions. \\
\textbf{Condition 1: Local Stationarity.}  $\theta(P_\eta) = \argmin_\theta R_{P_\eta}(\theta)$ is the minimizer of the risk with respect to a mixture distribution iff 
\begin{align}
\grad R_{P_{\eta}}(\theta(P_\eta)) & = (1 - \eta) \grad R_{P_\tparam}(\theta (P_\eta)) \nonumber \\
& + \eta \grad  R_{Q}(\theta(P_\eta)) = 0.
\end{align}
\textbf{Condition 2: Global Fit Optimality.} $\eSRM = \theta(P_{\eta^*})$ is the global minimizer of the population risk over all mixture distributions iff
\begin{align}
R_{P_{\eta^*}}(\theta(P_{\eta^*})) & = (1 - \eta^*) R_{P_0}(\theta(P_{\eta^*}))+ \eta^* R_{Q}(\theta(P_{\eta^*})) \nonumber \\ & \leq R_{P_\eta}(\theta(P_\eta)) \ \ \forall \eta \in \left[0,\frac{\epsilon}{1 - \epsilon} \right]
\end{align}
Using Conditions~1 and 2, we next derive the bias of SRM for mean estimation.

We make a few simple observations.
\bit
\item \textbf{Observation 1.} For any distribution $P$, we have,
\[ R_{P}(\theta) = \trace{\Sigma(P)} + \norm{\theta - \mu(P)}{2}^2 \]
\item \textbf{Observation 2.} Condition~1 reduces to, 
\[ \mu(P_\eta) = \theta_{\eta} = (1 - \eta) \mu(P) + \eta \mu(Q), \]
where $\mu(\cdot)$ is the Expectation functional.
\eit
\begin{lemma}\label{lem:risk_squaredloss}
Under the mixture model in Equation~\eqref{eqn:huber_mixture}, for the squared error, we have that,
\[ R_{P_\eta}(\theta_\eta) = \trace{\Sigma(P_\eta)} = (1 - \eta) \trace{\Sigma(\trueDist)} + \eta \trace{\Sigma(Q)} + \eta(1 - \eta)\norm{\mu(\trueDist)-\mu(Q)}{2}^2. \]
\end{lemma}

Now, from Lemma~\ref{lem:risk_squaredloss}, we know that
\[ R_{P_\eta}(\theta_\eta) = (1 - \eta) \trace{\Sigma(P)} + \eta \trace{\Sigma(Q)} + \eta(1 - \eta)\norm{\mu(P)-\mu(Q)}{2}^2 \]
As a function of $\eta$, $R_{P_\eta}(\theta_\eta)$ is a concave quadratic function. Hence, it is always minimized at the end points of the interval $\left[0,\epsilon/(1 - \epsilon) \right]$, which implies that $\eta^* \in \{0,\frac{\epsilon}{1 - \epsilon} \}$. \\ \\
Hence, we have that,
\[ \eSRM = \begin{cases}
    \theta_{\frac{\epsilon}{1 - \epsilon}}, & \text{if $R_{P_{\frac{\epsilon}{1- \epsilon}}}(\theta_{\frac{\epsilon}{1 - \epsilon}}) \leq R_{P_0}(\theta_0) $}.\\
   \tparam , & \text{otherwise}.
  \end{cases} \]
From Lemma~\ref{lem:risk_squaredloss}, $R_{P_{\frac{\epsilon}{1- \epsilon}}}(\theta_{\frac{\epsilon}{1 - \epsilon}}) \leq R_{P_0}(\theta_0)$ iff
\[ \left(1 - \frac{\epsilon}{1 - \epsilon}\right) \norm{\mu(P) - \mu(Q)}{2}^2 \leq \trace{\Sigma(P)} - \trace{\Sigma(Q)}  \]
Moreover, from Observation~2, we have that,
\[ \norm{\theta_{\frac{\epsilon}{1-\epsilon}} - \mu(P)}{2} = \frac{\epsilon} {1-\epsilon} \norm{\mu(P) - \mu(Q)}{2} \]
Combining the above two, we get that, 

\begin{align}
    \norm{\eSRM - \mu(P)}{2} = \left[\frac{\epsilon}{1 - \epsilon} \norm{\mu(P) - \mu(Q)}{2}\right]. \mathbf{1} \left\{ \norm{\mu(P) - \mu(Q)}{2}^2 \leq
\right. \nonumber \\ \left. \left( \frac{1 - \epsilon}{1 - 2 \epsilon} \right) (\trace{\Sigma(P)} - \trace{\Sigma(Q)}) \right\}.
\end{align}

Equation~\ref{eqn:thm:srm_mean2} follows from it.

\end{proof}

\subsubsection{Proof of Lemma~\ref{lem:risk_squaredloss}}
\begin{proof}
We give two alternate proofs of the Lemma.
\bit 
\item Proof 1: This proceeds by expanding on the definition of risk.
\begin{align*}
R_{P_\eta}(\theta_\eta) & = E_{z \sim P_\eta}[\norm{z - \theta_\eta}{2}^2] \\
&= (1 - \eta) E_{z \sim P_0}[\norm{z - \theta_\eta}{2}^2] + \eta E_{z \sim Q} [\norm{z - \theta_\eta}{2}^2]  \ \ \text{Expectation by conditioning.} \\
&= (1 - \eta) \left[ \trace{\Sigma(\trueDist)} + \norm{\theta_\eta - \mu(\trueDist)}{2}^2  \right] \\
& + \eta \left[ \trace{\Sigma(Q)} + \norm{\theta_\eta - \mu(Q)}{2}^2 \right] \ \ \text{From Observation 1.}
\end{align*}
Now, using Observation~2 we get that,
\[ \norm{\theta_\eta - \mu(Q)}{2} = (1 - \eta) \norm{\mu(\trueDist) - \mu(Q)}{2}\]
\[ \norm{\theta_\eta - \mu(\trueDist)}{2} = \eta \norm{\mu(\trueDist) - \mu(Q)}{2} \]
Plugging this into above, we get,
\begin{align*}
R_{P_\eta}(\theta_\eta) & = (1 - \eta) \trace{\Sigma(\trueDist)} + \eta \trace{\Sigma(Q)} + \norm{\mu(\trueDist) - \mu(Q)}{2}^2 \left( \eta^2 (1-\eta) + (1-\eta)^2 \eta \right)
\end{align*}
which recovers the statement of the Lemma.
\item Proof 2: This proceeds by Law of Total Variance, or the Law of Total Cummulants. We know that $R_{P_\eta} = \trace{\Sigma(P_\eta)}$. Let $Z \sim P_\eta$, and let $Y \sim \bern(1 - \eta)$ be the indicator if the sample is from the true distribution. Then $Z | Y = 1 \sim \trueDist$, while $Z | Y = 0 \sim Q$.  
\begin{align*}
\trace{\Sigma(P_\eta)} = \underbrace{(1 - \eta) \trace{\Sigma(\trueDist)} + \eta{\trace{\Sigma(Q)}}}_{\var(E[Z|Y])} + \underbrace{\eta (1 - \eta) \norm{\mu(\trueDist) - \mu(Q)}{2}^2}_{E[\var(Z|Y)]}. 
\end{align*}
\eit 
\end{proof}

\clearpage

\subsection{Proof of Lemma~\ref{lem:interval_p1D_huber}}
\begin{proof}
Let $P_\epsilon = (1- \epsilon)P^* + \epsilon Q$. Let $I^*$ be the interval $\mu \pm \frac{\sigma}{\delta_1^{\frac{1}{2k}}}$, where $\mu = \Exp_{x \sim P^*}[x]$. Moreover for notational convenience, let $f_n(u,v) = \sqrt{u(1-u)}\sqrt{\frac{\log(2/v)}{n}} + \frac{2}{3}\frac{\log(2/v)}{n}$. Let $\hat{I} = [a,b]$ be the interval obtained using $\calZ_1$, \ie the shortest interval containing $n(1-(\delta_1 + \epsilon + f_n(\epsilon+\delta_1,\delta_3)))$ points of $\calZ_1$. Note that in the algorithm, we have $\delta_1 = \epsilon$, and $\delta_3 = \delta/4$. As a first step, we bound the length of $\hat{I}$ and show that $\hat{I}$ and $I^*$ must necessarily intersect. 

\begin{claim}\label{claim:int_length}
Let $\hat{I}$ be the shortest interval containing $1 - \delta_4$ fraction of points, where $\delta_4 = (\delta_1 + \epsilon) + f_n(\epsilon + \delta_1,\delta_3) $. Further assume that $\delta_4 < \half$. Then with probability at least $1- \delta_3$,
\[ length(\hat{I}) \leq length(I^*) \leq \frac{2 \sigma }{\delta_1^{\frac{1}{2k}}}, \]
Moreover, if $\delta_4 < \half$, then $\hat{I} \cap {I^*} \neq {\phi}$, which implies
\[ |z - \mu| \leq \frac{4\sigma}{\delta_1^{\frac{1}{2k}}} \forall z \in \hat{I} \]
\end{claim}
\begin{proof}
We first show that with probability at least $1-\delta_3$, $I^*$ contains at least $n(1-\delta_4)$ points(Claim~\ref{claim:2}). Hence, since our algorithm chooses the shortest interval($\hat{I}$) containing ${1 - \delta_4}$ fraction of points, length of $\hat{I}$ is less than length of $I^*$. Next, if $\delta_4$ is less than $\half$, then there are two intervals $\hat{I}$ and $I^*$ respectively, which contain at least $n/2$ points. Hence, they must necessarily intersect.
\end{proof}

Next, we control the final error of our estimator. Let $|\hat{I}| = \sum_{z \in \calZ_2} \indic{z_i \in \hat{I}}$ be the number of points which lie in $\hat{I}$. Similarly, let $|\hat{I}_Q|$ and $|\hat{I}_{P^*}|$ number of points which lie in $\hat{I}$, which are distributed according to $Q$ and $P^*$ respectively.

\begin{align}\label{eqn:lem1_1}
    \abs{\frac{1}{|\hat{I}|} \sum \limits_{x_i \in \hat{I}} x_i - \mu}  \leq   \underbrace{\abs{\frac{1}{|\hat{I}|} \sum \limits_{\substack{x_i \in \hat{I} \\ x_i \sim Q}} (x_i - \mu)}}_{T1} + \underbrace{\abs{\frac{1}{|\hat{I}|} \sum \limits_{\substack{x_i \in \hat{I} \\ x_i \sim P^*}} (x_i - \mu)}}_{T2} 
\end{align}

\textbf{Control of T1.} To control T1, we can write it as:
\begin{align}\label{eqn:lem1_2}
    T1 & = \abs{\frac{1}{|\hat{I}|} \sum \limits_{\substack{x_i \in \hat{I} \\ x_i \sim Q}} (x_i - \mu)} \nonumber \\
    & \leq \underbrace{\frac{| \hat{I}_Q | }{| \hat{I} |}}_{T1a} \underbrace{\max\limits_{\substack{x_i \in \hat{I} \\ x_i \sim Q}} |x_i - \mu|}_{T1b}
\end{align}
where $\hat{I}_Q$ is the number of points in $\hat{I}$ distributed according to $Q$. To control T1a, we use Bernsteins inequality. To control T1b, we use Claim~\ref{claim:int_length}. The claim below formally controls T1.

\begin{claim}\label{claim:lem1_T1Control}
Let $\hat{I}$ be the shortest interval containing $n(1-\delta_4)$ of the points, where $\delta_4 = (\delta_1 + \epsilon) + f_n(\epsilon + \delta_1,\delta_3) $. Further assume that $\delta_4 < \half$. Then, with probability at least $1-\delta_3-\delta_5$, we have that,
\begin{align}\label{eqn:T1_bound}
    T1 \leq \frac{| \hat{I}_Q | }{| \hat{I} |} \max\limits_{\substack{x_i \in \hat{I} \\ x_i \sim Q}} |x_i - \mu| \leq \frac{\epsilon + f_n(\epsilon,\delta_5)}{1 - \delta_4} \frac{4 \sigma}{\delta_{1}^{1/2k}}
\end{align} 
\end{claim}
\begin{proof}
Using Bernstein's bound, we know that wp at least $1- \delta_5$, \[ | \hat{I}_Q | \leq n \paren{\epsilon + \sqrt{\epsilon(1-\epsilon)}\sqrt{\frac{\log(1/\delta_5)}{n}} + \frac{2}{3} \frac{\log(1/\delta_5)}{n} },  \]
    This follows from the fact that number of points drawn from Q which lie in $\hat{I}$ is less than the total number of points drawn according to Q.  In Claim~\ref{claim:int_length}, we showed that when $\delta_4 < \half$, then, with probability at least $1-\delta_3$,  we get that $\hat{I} \cap I^* \neq \phi$, \ie the intervals intersect, and that  $length(\hat{I}) < length(I^*)$. Hence, we get,
    \[ \max\limits_{\substack{x_i \in \hat{I} \\ x_i \sim Q}} |x_i - \mu| \leq \frac{4 \sigma}{\delta_1^{1/2k}} \]
\end{proof}

\textbf{Control of T2.} To control T2, we write it as 
\begin{align} T2 &=  \abs{\frac{| \hat{I}_{P^*} | }{|\hat{I}|} \sqprn{ \frac{1}{|\hat{I}_{P^*}|} \sum \limits_{\substack{x_i \in \hat{I} \\ x_i \sim P^*}} (x_i - \mu)}} \\
& \leq \frac{| \hat{I}_{P^*} | }{|\hat{I}|} \underbrace{\left| \paren{ \frac{1}{|\hat{I}_{P^*}|} \sum \limits_{\substack{x_i \in \hat{I} \\ x_i \sim P^*}}x_i } - E[x | x \in \hat{I},x\sim P^*] \right|}_{T2a} + \frac{| \hat{I}_{P^*} | }{|\hat{I}|} \underbrace{\left| E[x | x \in \hat{I},x\sim P^*] - \mu \right|}_{T2b}
\end{align}


\bit 
\item \textbf{Control of T2a:} To bound the distance between the mean of the points from $P^*$ within $\widehat{I}$ and $E[x|x\sim P^*,x \in \hat{I}]$, we will use Bernsteins bound(Lemma~\ref{lem:bernstein1D}) for bounded random variables. We know that the random variables are in a bounded interval $b = length(\widehat{I}) \leq \frac{\sigma}{\delta^{\frac{1}{2k}}}$, and that conditional variance of the random variables, when conditioned on them lying in $\hat{I}$ is controlled using Lemma~\ref{lem:variance_shift}. In particular, Lemma~\ref{lem:variance_shift} shows that for any event $E$, which occurs with probability $P(E) \geq \half$, 
\[ E_{x \sim P^*}[(x - E[x|x \in E])^2 | x \in E] \leq \sigma^2/P(E). \] Using these arguments, we get that with probability at least $1 - \delta_7$,
\begin{align}\label{eqn:lem1_t2a}
    T2a \leq \sqrt{\frac{2\sigma^2 (\log(3/\delta_7))}{P^*(\hat{I})|\hat{I}_{P^*}|}} + \frac{2 \sigma}{\delta_1^{1/2k}} \frac{\log (3/\delta_7)}{|\hat{I}_{P^*}|},
\end{align}
where $P^*(\hat{I})$ is the probability that a random variable drawn according to $P^*$ lies in $\hat{I}$. 

\item \textbf{Control of T2b:} To control $T2b$, we use the general mean shift lemma~(Lemma~\ref{lem:approx_mean}), which controls how far the mean can move when conditioned on an event. We get that,
\begin{equation}\label{eqn:lem1_t2b}
T2b \leq 2 \sigma \paren{P^*(\hat{I})^c}^{1 - 1/(2k)} 
\end{equation}
\eit 
Combining the bounds in \eqref{eqn:lem1_t2a} and \eqref{eqn:lem1_t2b}, we get
\begin{align}\label{eqn:lem_t2_ab}
    T2 \leq 2 \sigma \paren{P^*(\hat{I})^c}^{1 - 1/(2k)} + \sqrt{\frac{2\sigma^2 (\log(3/\delta_7))}{P^*(\hat{I})|\hat{I}_{P^*}|}} + \frac{2 \sigma}{\delta_1^{1/2k}} \frac{\log (3/\delta_7)}{|\hat{I}_{P^*}|}
\end{align}
Combining the upper bound on T1 in \eqref{eqn:T1_bound} with \eqref{eqn:lem_t2_ab}, we get that with probability at least $1 - \delta_3 - \delta_5 - \delta_6 - \delta_7$
\[ T1 + T2 \leq \frac{\epsilon + f_n(\epsilon,\delta_5)}{1 - \delta_4} \frac{4 \sigma}{\delta_{1}^{1/2k}} + 2 \sigma \paren{P^*(\hat{I})^c}^{1 - 1/(2k)} + \sqrt{\frac{2\sigma^2 (\log(3/\delta_7))}{P^*(\hat{I})|\hat{I}_{P^*}|}} + \frac{2 \sigma}{\delta_1^{1/2k}} \frac{\log (3/\delta_7)}{|\hat{I}_{P^*}|}  \]
We rearrange terms and use our assumption that $\epsilon$ is small enough that $\hat{I}_{P^*} \geq n/2$. We also plugin the upper bound on $\paren{P^*(\hat{I})^c}^{1 - 1/(2k)}$ from Claim~\ref{claim:lem1_p_control} and set $\delta_1 = \epsilon$, and $\delta_5 = \delta_6 = \delta_3 = \delta_7 = \delta/4$. Hence, we get that with probability at least $1 - \delta$
\begin{align}
    T1+ T2 \leq C_1 \sigma\epsilon^{1 - 1/{2k}} +  C_2 \sigma \paren{\frac{\log n}{n}}^{1 - \frac{1}{2k}} + C_3 \sigma \sqrt{\frac{\log (1/\delta)}{n}} + C_4 \sigma {\frac{\log(1/\delta)}{n \epsilon^{\frac{1}{2k}}}}
\end{align}
Since, we ensure that $\epsilon = \max \paren{ \epsilon, \frac{\log(1/\delta}{n}}$ hence, ${\frac{\log(1/\delta)}{n \epsilon^{\frac{1}{2k}}}} \leq \epsilon^{1 - \frac{1}{2k}}$. Note that our assumption of $\delta_4 < \half$ boils down to $\epsilon$ being small enough such that $2 \epsilon + \sqrt{\epsilon\frac{\log(4/\delta)}{n}} + \frac{\log(4/\delta)}{n} < \half$.
Hence, we recover the final statement of the theorem.

\end{proof}

\subsubsection{Auxillary Proofs}

\begin{claim}\label{claim:lem1_p_control}
Let $\hat{I}$ be the shorted interval containing $n(1-\delta_4)$ points from $\calZ_1$. Let $P^*(\hat{I})$ is the probability that a random variable drawn according to $P^*$ lies in $\hat{I}$. Then, there exists universal constants $C_1,C_2 > 0$ such that wp at least $1 - \delta_6$, we have that
\end{claim}
\begin{align}
    \paren{P^*(\hat{I})^c}^{1 - \frac{1}{2k}} \leq C_1 \epsilon^{1 - \frac{1}{2k}} + C_2 \delta_1^{1 - \frac{1}{2k}} + C_3 \paren{\frac{\log n}{n}}^{1 - \frac{1}{2k}} + C_4 \paren{\frac{\log (1/\delta_6)}{n}}^{1 - \frac{1}{2k}} + C_5 \paren{\frac{\log (1/\delta_3)}{n}}^{1 - \frac{1}{2k}}
\end{align}
\begin{proof}
 Note that $\hat{I}$ is obtained by choosing the shortest interval containing $n(1-\delta_4)$ points from $\calZ_1$. We first bound $P^*_n(\hat{I})$, \ie the empirical probability of samples distributed according to $P^*$ which lie in $\hat{I}$. To do this, note that in $\calZ_1$, number of points drawn from Q which lie in $\hat{I}$, say $\hat{n}_Q$ is less than the total number of points drawn according to Q. Using Bernstein's bound, we know that wp at least $1- \delta_6$, \[ | \hat{n}_Q | \leq n \paren{\epsilon + \sqrt{\epsilon(1-\epsilon)}\sqrt{\frac{\log(1/\delta_6)}{n}} + \frac{2}{3} \frac{\log(1/\delta_6)}{n} }\]
 Let $\hat{n}_{P^*}$ be the number of points in $\calZ_1$, which are drawn from $P^*$ and which lie in $\hat{I}$. Since $|\hat{n}_Q| + |\hat{n}_{P^*}| = |\hat{I}| = n(1 - \delta_4)$, hence the above implies that with probability at least $1 - \delta_6$, \[ | \hat{n}_{P^*}| \geq n(1-\delta_4) - n \paren{\epsilon + \sqrt{\epsilon(1-\epsilon)}\sqrt{\frac{\log(1/\delta_6)}{n}} + \frac{2}{3} \frac{\log(1/\delta_6)}{n} }, \]
Note that $P^*_n(\hat{I}) = \frac{|\hat{n}_{P^*}|}{\sum_{i} \indic{x_i \sim P^*}}$. Hence, we get that,
\begin{align} 
P^*_n(\hat{I})  & \geq  \frac{|\hat{n}_{P^*}|}{n} \nonumber \\
 & \geq 1 - (\epsilon + \delta_4) - f_n(\epsilon,\delta_6)
\end{align}
This implies that,
\begin{align} 
P^*_n(\hat{I})^c & \leq (\epsilon + \delta_4) + f_n(\epsilon,\delta_6)   \nonumber \\
& \leq 2 \epsilon + \delta_1 + f_n(\epsilon,\delta_6) + f_n(\epsilon + \delta_1,\delta_3) \nonumber \\
& \leq 4\epsilon +  2\delta_1 + C_1 \frac{\log(1/\delta_6)}{n} +  C_2 \frac{\log(1/\delta_3)}{n}
\end{align}
To finally bound the probability of a sample drawn from $P^*$ to lie in $\hat{I}$, we use the relative deviations VC bound(Lemma~\ref{lem:rel_vc}), which gives us,
 \begin{align}
        P^*(\hat{I})^c \leq \underbrace{P^*_n(\hat{I})^c}_{A_1} + 4 \sqrt{\paren{\frac{ P^*_n(\hat{I})^c \log \calS[2n]}{n}} + \paren{\frac{P^*_n(\hat{I})^c \log(4/\delta_6)}{n}}} + \frac{\log \calS[2n]}{n} + \frac{\log(4/\delta_6)}{n}
    \end{align}
where $\calS[2n] = O(n^2)$. Using that $\sqrt{ab} \leq a + b, \forall a,b \geq 0$, we get that,
\begin{align}
        P^*(\hat{I})^c \leq C_1 {P^*_n(\hat{I})^c} + C_2 \paren{ \frac{\log \calS[2n]}{n} + \frac{\log(4/\delta_6)}{n}}
    \end{align}
Hence, we get that,
\begin{align}
     \paren{P^*(\hat{I})^c}^{1 - \frac{1}{2k}} \leq C_1 \epsilon^{1 - \frac{1}{2k}} + C_2 \delta_1^{1 - \frac{1}{2k}} + C_3 \paren{\frac{\log n}{n}}^{1 - \frac{1}{2k}} + C_4 \paren{\frac{\log (1/\delta_6)}{n}}^{1 - \frac{1}{2k}} + C_5 \paren{\frac{\log (1/\delta_3)}{n}}^{1 - \frac{1}{2k}}
\end{align}
\end{proof}

\begin{claim}\label{claim:1}
Let $P^*(I^*)$ be the probability that a sample drawn according from $P_{\epsilon}$ is distributed according to $P^*$ and lies in $I^*$.  
\[ P^*(I^*) \geq (1 - \epsilon)(1 - \delta_1) = 1 - ( \epsilon + \delta_1 - \epsilon \delta_1) \geq 1 - \underbrace{( \epsilon + \delta_1)}_{\delta_2} = 1 - \delta_2 \]
\end{claim}
\begin{proof}
For any $x \sim P_\epsilon$, define, $z_i = 1$ if $x \sim P^*$. Now, for any $x \sim P^*$, we know that, by chebyshevs we know that, 
\[ P(|x - \mu | \geq t) = P( (x - \mu)^{2k} \geq t^{2k}) \leq E[(x - \mu)^{2k}]/t^{2k} \leq C_{2k}\sigma^{2k} / t^{2k}  \]
Hence, we get that wp at least $1 - \delta_1$, $x \in \mu \pm \sigma/(\delta_1)^{1/2k}$
\end{proof}

The following claim lower bounds the empirical fraction of samples which are distributed according to $P^*$ and lie in $I^*$, when $n$ samples are drawn from $P_\epsilon$. 

\begin{claim}\label{claim:2}
Let $P^*_n(I^*)$ be the empirical fraction of points which are distributed according to $P^*$ and lie in $I^*$, when $n$ samples are drawn from $P_\epsilon$. Then, with probability at least $1 - \delta_3$, 
\[ P^*_n(I^*) \geq 1 - \underbrace{\paren{\delta_2 + \sqrt{(\delta_2 (1 - \delta_2))} \sqrt{\frac{\log(1/\delta_3)}{n}} + \frac{2}{3} \frac{\log(1/\delta_3)}{n}}}_{\delta_4 = (\delta_1 + \epsilon) + f_n(\epsilon + \delta_1,\delta_3)}, \]
\end{claim}
\begin{proof}
This follows from Bernstein's inequality(Lemma~\ref{lem:bernstein1D}).
\end{proof}


\begin{lemma}\label{lem:bernstein1D}[Bernsteins bound,] Let $X \sim P^*$ be a scalar random variable such that $|X - E[x] | \leq b$ with variance $\sigma^2$. Then, given $n$ samples $\{x_1,x_2,\ldots,x_n \} \sim P^*$, the empirical mean, $\bar{x_n} = \frac{1}{n} \sum \limits_{i=1}^n x_i $ is such that,
\[ P(|\bar{x_n} - E[x] | > t) \leq 2 \exp \paren{\frac{- nt^2}{2 \sigma^2 + 2bt/3}} \]
which can be equivalently re-written as. With probability at least $1- \delta$, 
\[ | \bar{x_n} - E[x] | \leq \sqrt{\frac{2\sigma^2 \log(1/\delta)}{n}} + \frac{2 b \log(1/\delta) }{3 n} \]

\end{lemma}

\begin{lemma}\label{lem:rel_vc}[Relative deviations,~\citep{vapnik2015uniform}] Let $\calF$ be a function class consisting of binary functions $f$. Then, with probability at least $1- \delta$, 
\[ \sup_{f \in \calF} |P(f) - P_n(f)| \leq 4 \sqrt{P_n(f) \frac{ \log(S_{\calF}(2n)) + \log(4/\delta)}{n}} +  C_1 \frac{\log(S_{\calF}(2n)) + \log(4/\delta)}{n}, \]
where $S_\calF(n) = \sup \limits_{z_1,z_2,\ldots,z_n} | \{ (f(z_1),f(z_2),\ldots,f(z_n)) : f \in \calF\}|$ is the growth function, \ie the maximum number of ways into which $n$-points can be classified the function class.
\end{lemma}

\begin{lemma}\label{lem:approx_mean}[General Mean shift,~\citep{steinhardt2018robust}]  Suppose that a distribution $P^*$ has mean $\mu$ and variance $\sigma^2$ with bounded $2k^{th}$-moments. Then, for any event $A$ which occurs with probability at least $1 - \epsilon \geq \half$,
\[ |\mu - E[x|A] | \leq 2 \sigma \epsilon^{1 - \frac{1}{2k}} \]
In particular, for just bounded second moments, we get that $|\mu - E[x|A] | \leq 2 \sigma \sqrt{\epsilon}$.
\end{lemma}

\begin{proof}
For any event E, Let $\indic{E}$ denote the indicator variable for $E$. 
\begin{align}
    |E_{x \sim P^*} [x |E ] - \mu]| = \frac{| E_{x \sim P^*}((x - \mu ) \indic{E}) |}{P(E)} \leq \frac{ E[| x - \mu |^{p}]^{\frac{1}{p}} (E[\indic{E}^{q}]^{1/q})}{{P(E)}},
\end{align}
where $p,q > 1$ are such that $1/p + 1/q = 1$. Put $p = 2k$, we get,
\[|E_{x \sim P^*} [x |E ] - \mu]| \leq \frac{\sigma}{(P(E))^{1/2k}} \]
Now, we know that, $|E[X|A] - \mu | = \frac{1 - P(A)}{P(A)} | E[X|A^c] - \mu |$. Putting $E = A^c$, we get, 
\[ |E[X|A] - \mu | \leq \frac{1 - P(A)}{P(A)}  \frac{\sigma}{(1 - P(A))^{1/2k}}  \leq 2 \sigma \epsilon^{(1 - \frac{1}{2k})}.  \]
\end{proof}

\begin{lemma}\label{lem:variance_shift} [Conditional Variance Bound] Suppose that a distribution $P^*$ has mean $\mu$ and variance $\sigma^2$. Then, for any event $A$ which occurs with probability at least $1 - \epsilon$, the variance of the conditional distribution is bounded as:
\[ (E[(x - E[x|A])^2 | A]) \leq \frac{\sigma^2}{ (1  - \epsilon)} \]
\end{lemma}
\begin{proof}
Let $\mu_A = E[y|A]$, $d = \mu_A - \mu$. From Lemma~\ref{lem:approx_mean}, we know, $d \leq \sigma 2\sqrt{{\epsilon}}$. Observe the following,
\begin{align} 
E[(y - \mu_A)^2 | A ] = E[(y - \mu - d)^2 | A] &= E[ \paren{(y - \mu)^2 - 2d(y - \mu) + d^2 }| A ] \\
& = E[(y - \mu)^2 | A] - d^2 \\
& \leq E[(y - \mu)^2 | A]  \\
& \leq \frac{\sigma^2}{1 - \epsilon},
\end{align}
\end{proof}

\clearpage

\subsection{Proof of Lemma~\ref{lem:ppEst}}
\begin{proof}
Let $\eparam_{\delta} = \inf \limits_{\theta} \sup \limits_{u \in \calN^{1/2}(\calS^{p-1})} | u^T\theta - \text{INTERVAL1D}(\{u^T x_i\}_{i=1}^n,\epsilon,\frac{\delta}{5^p}) |$ and let $\tparam = \Exp[x]$ be the true mean. Then, we can write the $\ell_2$ error in its variational form.
\begin{align}
    \norm{\eparam_{\delta} - \tparam}{2} = \sup \limits_{u \in  \calS^{p-1}} |u^T(\eparam_{\delta} - \tparam)| 
\end{align}

Suppose $\{y_i\}$ is a $\half$-cover of the net, so there exist a $y_j$ such that $u = y_j + v$, where $\norm{v}{2} \leq \epsilon$.  
\begin{align*}
    \norm{\eparam_{\delta} - \tparam}{2} & \leq  \sup \limits_{u \in \calS^{p-1}} |y_j^T (\eparam_{\delta} - \tparam)| +  |v^T (\eparam_{\delta} - \tparam)| \\
    & \leq \sup \limits_{y_j \in \calN^{\half}\paren{ \calS^{p-1}}} |y_j^T (\eparam_{\delta} - \tparam)| +  \norm{v}{2} \norm{\eparam_{\delta} - \tparam}{2} \\
    & \leq 2 \sup \limits_{y_j \in \calN^{\half}\paren{ \calS^{p-1}}} |y_j^T (\eparam_{\delta} - \tparam)|
    \end{align*}
    Let $f(u^T P_n,\epsilon;\tilde{\delta}) = \text{INTERVAL1D}(\{u^T x_i\}_{i=1}^n,\epsilon,\frac{\delta}{5^p})$
\begin{align}
  \norm{\eparam_{\delta} - \tparam}{2} &\leq 2 \sup \limits_{u \in \calN^{1/2}} | u^T(\eparam - \tparam)| \\
  & \leq 2 \sqprn{\sup \limits_{u \in \calN^{1/2}} | u^T\eparam - f(u^T P_n,\epsilon;\tilde{\delta})| + \sup \limits_{u \in \calN^{1/2}} | u^T\tparam - f(u^T P_n,\epsilon;\tilde{\delta})| } \\
  & \leq 4 \sup \limits_{u \in \calN^{1/2}} | u^T\tparam - f(u^T P_n,\epsilon;\tilde{\delta})|
\end{align}

For a fixed $u$, the distribution $u^TP$ has mean $u^T\tparam$, where $\tparam$ is the mean of the multivariate distribution $P$. Hence, we get that, for a confidence level $\tilde{\delta}$, when the interval estimator is applied to the projection of the data long u, it returns a real number such that, with probability at least $1 - \tilde{\delta}$ 

\[ | f(u^T P_n; \epsilon; \tilde{\delta}) - u^T \tparam | \leq C_1 \sigma_u \max\paren{2\epsilon,\frac{\log(1/\tilde{\delta})}{n}}^{1 - 1/{2k}} +  C_2 \sigma_u \paren{\frac{\log n}{n}}^{1 - \frac{1}{2k}} + C_3 \sigma_u \sqrt{\frac{\log (1/\tilde{\delta})}{n}} \]

Note that $\sigma_u = \sqrt{u^T \Sigma u} \leq \norm{\Sigma}{2}^{\half}$. Taking a union bound over the elements of the cover, and using the fact that $|\calN^{1/2}(\calS^{p-1})| \leq 5^p$~\citep{wainwright2019high}, we substitute $\tilde{\delta} = \delta/(5^p)$ and recover the statement of the Lemma.

\end{proof}

\subsection{Proof of Corollary~\ref{cor:ppEst_sparse}}
\begin{proof}
Let $\eparam_{s} = \inf \limits_{\theta \in \Theta_s} \sup \limits_{u \in \calN^{1/2}_{2s}(\calS^{p-1})} | u^T\theta - \text{INTERVAL1D}(\{u^T x_i\}_{i=1}^n,\epsilon,\frac{\delta}{9^p}) |$. Observe that since $\eparam_{s}$ and the true mean $\tparam$ are both $s$-sparse. Hence, the error vector $\eparam - \tparam$ is atmost $2s$-sparse. Then, we can write the $\ell_2$ error in its variational form,
\begin{align}
    \norm{\eparam_{\delta} - \tparam}{2} = \sup \limits_{u \in  \calS^{p-1} \cap \calB_{2s}} |u^T(\eparam_{\delta} - \tparam)|,
\end{align}
where $ \calS^{p-1} \cap \calB_{2s}$ is the set of unit vectors which are $2s$-sparse. The remaining of the proof follows along the lines of proof of Lemma~\ref{lem:ppEst}, coupled with the fact that the cardinality of the half-cover of an $2s$-sparse ball, \ie $\abs{\calN^{\half}(\calS^{p-1})} \leq \paren{\frac{6ep}{s}}^s$~\citep{vershynin2009role}.

\end{proof}